\documentclass{article}

\usepackage[sort, numbers]{natbib}

\PassOptionsToPackage{numbers, compress}{natbib}

\usepackage[preprint]{neurips_2024}

\usepackage[utf8]{inputenc} 
\usepackage[T1]{fontenc}    
\usepackage{hyperref}       
\usepackage{url}            
\usepackage{booktabs}       
\usepackage{amsfonts}       
\usepackage{nicefrac}       
\usepackage{microtype}      
\usepackage{xcolor}         
\usepackage{amsthm}
\usepackage{amsmath}
\usepackage{graphicx}
\usepackage{amssymb}
\usepackage{lipsum}
\usepackage{multirow}
\usepackage{makecell}
\usepackage{longtable}
\usepackage{bbm}
\usepackage{svg}

\theoremstyle{plain}
\newtheorem{theorem}{Theorem}

\newtheorem{lemma}{Lemma}

\newtheorem{axiom}{Axiom}

\title{Quantifying In-Context Reasoning Effects and Memorization Effects in LLMs}

\author{%
  Siyu Lou \\
  Shanghai Jiao Tong University\\
  Eastern Institute of Technology, Ningbo\\
  \texttt{siyu.lou@sjtu.edu.cn} \\
   \And
   Yuntian Chen \\
  Eastern Institute of Technology, Ningbo \\
  \texttt{ychen@eitech.edu.cn} \\
     \AND
   Xiaodan Liang \\
  Sun Yat-Sen University\\
  \texttt{xdliang328@gmail.com} \\
       \And
 Liang Lin \\
  Sun Yat-Sen University\\
  \texttt{linlng@mail.sysu.edu.cn} \\
         \And
 Quanshi Zhang\thanks{corresponding author} \\
  Shanghai Jiao Tong University\\
  \texttt{zqs1022@sjtu.edu.cn} \\
}

\begin{document}

\maketitle

\begin{abstract}
In this study, we propose an axiomatic system to define and quantify the precise memorization and in-context reasoning effects used by the large language model (LLM) for language generation. These effects are formulated as non-linear interactions between tokens/words encoded by the LLM. Specifically, the axiomatic system enables us to categorize the memorization effects into foundational memorization effects and chaotic memorization effects, and further classify in-context reasoning effects into enhanced inference patterns, eliminated inference patterns, and reversed inference patterns. Besides, the decomposed effects satisfy the sparsity property and the universal matching property, which mathematically guarantee that the LLM's confidence score can be faithfully decomposed into the memorization effects and in-context reasoning effects. Experiments show that the clear disentanglement of memorization effects and in-context reasoning effects enables a straightforward examination of detailed inference patterns encoded by LLMs.
\end{abstract}

\section{Introduction}\label{sec:intro}
It is a long-standing issue whether a large language model (LLM) conducts inference based on both memorized knowledge and the reasoning logic. There have been many empirical studies of analyzing the reasoning or memorization effects in the inference of an LLM. For example, previous studies usually tested an LLM's reasoning capacity on specific reasoning tasks~\cite{yu2019reclor,nie2020adversarial,tian2021diagnosing,liu2021logiqa}, or measuring memorization capacity by applying the question-answering task on common-sense questions~\cite{feldman2020neural,carlini2022quantifying,zhang2024counterfactual}. 

However, going beyond the empirical testing, how to formulate and quantify the exact reasoning effects and memorization effects in mathematics used by an LLM is still an open problem. It is because there is no clear boundary between the memorization and reasoning, and the reasoning logic itself is also memorized by the LLM.  \textbf{Therefore, in this study, we only focus on the in-context reasoning and context-agnostic memorization with relatively clear boundaries}, as follows. Given an input prompt $\mathbf{x}=[x_1,x_2,\dots, x_n]$, let us divide the prompt $\mathbf{x}$ into an in-context premise prompt $\mathbf{x}_{\text{p}}$ and a question prompt $\mathbf{x}_{\text{q}}$, \emph{e.g.}, $\mathbf{x}_{\text{p}} = \{\textit{John and Mary are married.}\}$, $\mathbf{x}_{\text{q}}=\{\textit{James is the son of John. Mary is the mother of}\}$. Let us use $v(x_{n+1} | \mathbf{x})=\log p(x_{n+1}|\mathbf{x})/ \left(1-p(x_{n+1}|\mathbf{x})\right)$ to measure the confidence score, where $p(x_{n+1}|\mathbf{x})$ to denote the probability of generating the next token $x_{n+1}$. When we remove the premise to break the reasoning chain, the language generation on the question $v_{\text{\rm mem}}(x_{n+1}|\mathbf{x})  = v(x_{n+1}|\mathbf{x}_{\text q})$ can be regarded as the inference result based on context-agnostic memorization when only the question is provided as input. Then, the output difference $v_{\text{\rm rsn}}(x_{n+1}|\mathbf{x}) = v(x_{n+1}|\mathbf{x} = \{\mathbf{x}_{\text{p}} , \mathbf{x}_{\text q}\}) - v(x_{n+1}|\mathbf{x}_{\text q})$ with the premise is owing to the in-context reasoning effect. 

Instead of simply measuring the overall reasoning utility $v_{\text{\rm rsn}}(x_{n+1}|\mathbf{x})$ and memorization utility $v_{\text{\rm mem}}(x_{n+1}|\mathbf{x})$, as Figure~\ref{fig:fig1} shows, \textbf{we hope to propose a new axiom system to decompose these utilities into a set of reasoning patterns and memorization patterns,} which can be considered as a more detailed and rigorous explanation of the underlying inference patterns encoded by the LLM. 

Although it is counterintuitive, \citet{li2023does} have universally observed and \citet{ren2024where} have mathematically proven that given an input prompt $\mathbf{x}$, the confidence score $v(x_{n+1} | \mathbf{x})$ of a neural network can be be disentangled into tens/hundreds of inference patterns, \emph{i.e.},$v(x_{n+1}|\mathbf{x}) = \sum_{S\subseteq \Omega} I(S|\mathbf{x})$. Here, each inference pattern is formulated as an interaction $I(S|\mathbf{x})$ to represent a non-linear relationship between different input tokens. As Figure~\ref{fig:fig1} shows, we can regard an interaction as a ``phrase'' automatically learned by the LLM. More importantly, \textbf{a set of desirable properties have been proven to demonstrate that these tens/hundreds of interactions can faithfully explain all detailed inference patterns that determine the final scalar output $v(x_{n+1} | \mathbf{x})$ of the LLM} (\emph{e.g.}, sparsity, universal matching, see Theorems~\ref{thm:upperbound} and~\ref{thm:universalMatching})~\cite{li2023does, ren2024where, ren2023defining, zhou2024explaining, shen2023can}.   

\begin{figure}
    \centering
    \includegraphics[clip, trim=0cm 7.7cm 0cm 5.5cm, width=0.95\textwidth]{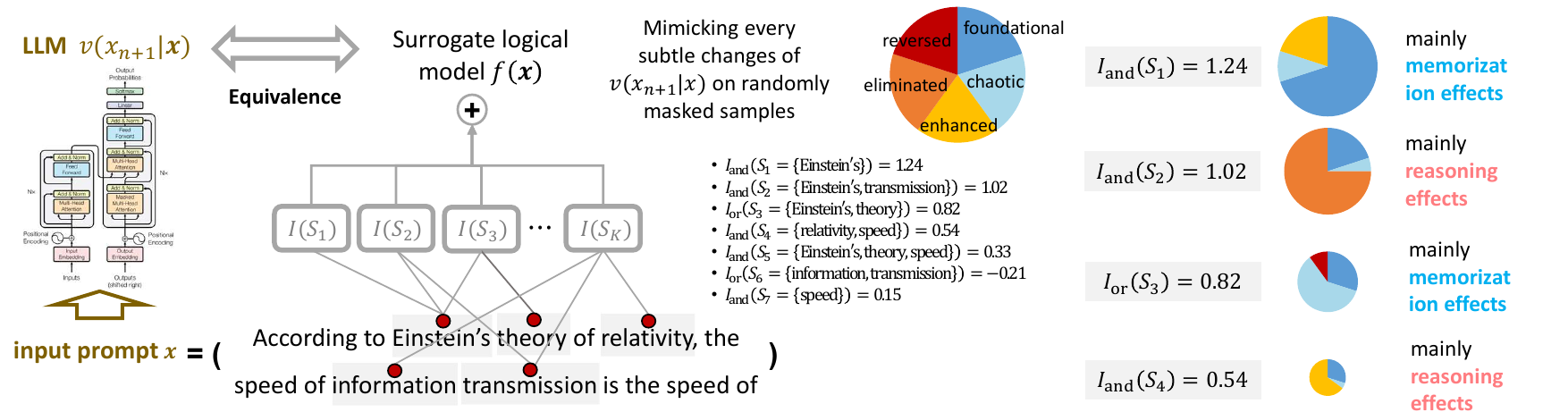}
    \caption{Illustration of interactions encoded by the LLM and the decomposition of their effects into memorization and in-context reasoning effect. The LLM's confidence score can be decomposed into a few interactions, such as $S_2 = \{\textit{Einstein's}, \textit{transmission}\}$. Each interaction has a numerical effect $I(S)$ to the LLM's confidence score. Then, we can decompose the interaction effect into the memorization effects (foundational memorization effect and chaotic memorization effect) and the in-context reasoning effect. The in-context reasoning effect can be further categorized as the enhanced inference pattern, the eliminated inference pattern or the reversed inference pattern.}
    \label{fig:fig1}
\end{figure}

Therefore, we further categorize these interactions, and \textbf{we find that some interactions mainly represent memorization effects, while other interactions correspond to in-context reasoning effects}, as Figure~\ref{fig:fig1} shows. Specifically, we propose two axioms to define exact interaction effects for memorization and interaction effects for in-context reasoning, \emph{i.e.}, the axiom of condition dependence and the axiom of variable independence. Based on the axiomatic system, we classify the interaction effects into the following three primary categories.

\textbf{$\bullet$ Foundational memorization effect.} When the input prompt only contains the question $\mathbf{x}_q$, the interaction effects encoded by the LLM can be considered as memorization effects.

\textbf{$\bullet$ In-context reasoning effect.} When we add the premise $\mathbf{x}_p$ into the question $\mathbf{x}_q$, \emph{i.e.}, $\mathbf{x} = \{\mathbf{x}_p, \mathbf{x}_q\}$, the interaction effect changed by the addition of $\mathbf{x}_p$ are defined as in-context reasoning effects. In particular, these in-context reasoning effects can be further categorized into three types: enhanced inference patterns, eliminated inference patterns, and reversed inference patterns.   

\textbf{$\bullet$ Chaotic memorization effect.} Besides in-context reasoning effects, the inclusion of the premise $\mathbf{x}_p$ also introduces additional effects to the confidence score. Such effects are not directly associated with reasoning, and they are termed chaotic memorization effects.

The above fine-grained categorization of interaction effects provides new insights into the inference of an LLM. Given an input prompt, the LLM usually uses foundational memorization effects of different complexities for inference. The complexity of an interaction effect is quantified as the number of input variables involved in the interaction. In comparison, in-context reasoning mainly has two utilities. First, in-context reasoning effects offset almost all complex memorization effects, \emph{i.e.}, reasoning effects triggered by the premise help the LLM remove incorrect memorization effects. Besides, in-context reasoning also brings in some new relative simple interaction effects. 

In addition, our experiments have shown that the clear disentanglement of memorization effects and in-context reasoning effects enables a straightforward examination of detailed inference patterns encoded by LLMs. We have discovered that the LLM actually uses lots of incorrect interactions for inference, although the final output of the LLM still looks reasonable.

\section{Axiomatic system for memorization effects and reasoning effects} \label{sec:axiomaticsystem_full}
\subsection{Preliminaries: interactions}

\textbf{Defining and decomposing AND-OR interactions.} Given a DNN $v$ and an input sample $\mathbf{x}$ with $n$ variables, indexed by $N = \{1,2,\dots, n\}$, we use $v(\mathbf{x})$ to denote the scalar output of the DNN. We can choose different settings for $v(\mathbf{x})$. For example, When an LLM generates the next token $x_{n+1}$ based on previous prompt $\mathbf{x}$, we can set $v(\mathbf{x}) = \log p(x_{n+1}|\mathbf{x})/(1-p(x_{n+1}|\mathbf{x}))$, where $p(x_{n+1}|\mathbf{x})$ denotes the probability of generating $x_{n+1}$. \citet{chen2024defining, zhou2023explaining} have proven that a DNN can be considered to equivalently encode a set of AND interactions $\Omega_{\text{\rm and}}$ and a set of OR interactions $\Omega_{\text{\rm or}}$. It has also been proven that the output score $v(\mathbf{x})$ can be decomposed into numerical effects of these interactions, as follows.
\begin{equation}
    v(\mathbf{x}) = \sum\nolimits_{S\in \Omega_{\text{\rm and}}}I_{\text{\rm and}}(S|\mathbf{x}) +  \sum\nolimits_{S\in \Omega_{\text{\rm or}}}I_{\text{\rm or}}(S|\mathbf{x}) +v(\mathbf{x}_\emptyset),
\end{equation}
where $v(\mathbf{x}_\emptyset)$ denotes an input sample, in which all input variables are masked\footnote{When masking a word, we replace the embeddings of all related tokens with its baseline embedding~\cite{ren2023can}.\label{footnote:choosewords}}.

Physical meanings of AND-OR interactions can be understood as follows. Theorem~\ref{thm:universalMatching} shows that each AND interaction can be considered as the AND relationship between the input variables in $S$ encoded by the LLM. In LLMs, we can select either words or tokens from the prompt as input variables\footnote{For longer prompts, we may choose key words with semantic significance, excluding insignificant words like articles, prepositions, and conjunctions.}. For example, given an input prompt $\mathbf{x}
=[\textit{The}, \textit{student}, \textit{found}, \textit{the}, \textit{project},  \textit{to}, \textit{be}, \textit{a}, \textit{piece},\textit{of}, \textit{cake}]$, the LLM encodes the non-linear relationship between the words in $S\!=\!\{\textit{piece}, \textit{of}, \textit{cake}\}$ to make a numerical effect towards the meaning of ``easy'' on the network output. Otherwise, Theorem~\ref{thm:universalMatching} will show that the effect $I(S|\mathbf{x}_{\text{\rm masked}})$ will be removed from the output $v(\mathbf{x}_{\text{\rm masked}})$, when we mask any word in $S$ in the input (obtaining $\mathbf{x}_{\text{\rm masked}}$). Similarly, the OR interaction $I_{\text{or}}(S|\mathbf{x})$ corresponds to the OR relationship between the words in $S$. For example, given a prompt $\mathbf{x}=[\textit{The}, \textit{girl}, \textit{is}, \textit{happy}, \textit{and}, \textit{excited}]$, the presence of either word in $S=\{\textit{happy}, \textit{excited}\}$ will contribute to the positive sentiment to the network output. We can compute $I_{\text{and}}(S|\mathbf{x})$ and $I_{\text{or}}(S|\mathbf{x})$ as follows.
\begin{equation}\label{eq:andor}
     I_{\text{and}}(S|\mathbf{x}) = \sum_{T\subseteq S, S\neq \emptyset}(-1)^{|S|-|T|}v_{\text{\rm and}}(\mathbf{x}_T), ~
     I_{\text{or}}(S|\mathbf{x}) = -\sum_{T\subseteq S, S\neq \emptyset}(-1)^{|S|-|T|}v_{\text{\rm or}}(\mathbf{x}_{N\setminus T}), 
\end{equation}
where $\mathbf{x}_T$ denotes a masked sample. Embeddings of words/tokens in $N\setminus T$ are masked in the sample $\mathbf{x}_T$. AND-OR interactions are extracted by decomposing $\forall T\subseteq N, ~v(\mathbf{x}_T)$ into the output component for AND interactions $v_{\text{\rm and}}(\mathbf{x}_T)=0.5v(\mathbf{x}_T) + \theta_T$ and the output component for OR interactions $v_{\text{\rm or}}(\mathbf{x}_T)=0.5v(\mathbf{x}_T) - \theta_T$, subject to $v(\mathbf{x}_T) = v_{\text{\rm and}}(\mathbf{x}_T) + v_{\text{\rm or}}(\mathbf{x}_T) + v(\mathbf{x}_\emptyset)$ . We can prove\footnote{Please see Appendix~\ref{appendix:poofforvcomponent} for details.} that the output component $v_{\text{\rm and}}(\mathbf{x}_T)$ only contains AND interactions, \emph{i.e.}, $v_{\text{\rm and}}(\mathbf{x}_T) = \sum_{S\subseteq T, S\neq \emptyset}I_{\text{and}}(S|\mathbf{x})$, and the component $v_{\text{\rm or}}(\mathbf{x}_T)$ only contains OR interactions, \emph{i.e.}, $v_{\text{\rm or}}(\mathbf{x}_T)=\sum_{S\cap T\neq \emptyset, S\neq \emptyset}I_{\text{or}}(S|\mathbf{x})$. The parameter  $\theta_T$ is learned to obtain the sparsest AND-OR interactions~\cite{li2023defining}, \emph{i.e.}, $\text{min}_{\{\theta_T\}}\|[I_{\text{and}}(S_1|\mathbf{x}), \dots, I_{\text{and}}(S_{2^n}|\mathbf{x})]^\top\|_1 +  \|[I_{\text{or}}(S_1|\mathbf{x}), \dots, I_{\text{or}}(S_{2^n}|\mathbf{x})]^\top\|_1$. Please refer to Appendix~\ref{append:extract} for details.

\textbf{How faithfulness of the interaction-based explanation is guaranteed.} The counterintuitive conclusion that the inference logic of a DNN can be faithfully explained as symbolic patterns (interactions) is guaranteed by the following four properties. 

$\bullet$ \textit{Sparsity property. }The sparsity property means that a DNN for classification\footnote{Generating the next token based on the input prompt is implemented as a classification task.} usually just encodes a small number of salient interactions. 
It has been widely observed~\cite{ren2023defining, li2023does, shen2023can} and mathematically proven by Theorem~\ref{thm:upperbound}~\cite{ren2024where}. \emph{I.e.}, among all $2^n$ possible interactions, only approximately $\mathcal{O}(n^{\kappa}/\tau)$ interactions have salient interaction effects. Please see Appendix~\ref{appendix:sparsity} for more discussions. In fact, this conclusion can be extended to OR interactions, because an OR interaction can be regarded as a specific AND interaction\footnote{Please see Appendix~\ref{appendix:or-interaction} for details.\label{footnote:baseline}}.

$\bullet$ \textit{Universal matching property.} Theorem~\ref{thm:universalMatching}~\cite{zhou2023explaining} has proven that we can use the aforementioned small number of salient interactions to well approximate the output of a DNN on all $2^n$ masked samples.

$\bullet$ \textit{Transferability property.} The transferability of interactions has been widely observed on various DNNs for different tasks~\cite{li2023does}. \emph{I.e.}, different DNNs trained for the same task usually encode similar interactions for inference. Besides, a small set of common interactions are usually frequently extracted from different input samples, and the DNN mainly uses these common interactions for inference.

$\bullet$ \textit{Discrimination property.} It has been discovered that interactions have discrimination power~\cite{li2023does}, that is when an interaction is extracted from different samples, it usually pushes these samples towards the classification of a certain category. 

\begin{theorem}[Sparsity property, proved by~\cite{ren2024where}]\label{thm:upperbound}
    Let us be given an input sample $\mathbf{x}=[x_1,\dots,x_n]$ with $n$ input variables, and a DNN with smooth\footnote{\citet{ren2024where} have proposed three conditions to define a DNN with smooth output score on different masked samples. These conditions can be briefly summarized as follows. (1) the DNN does not encode extremely complex interactions. (2) Let us compute the average classification confidence when we mask different random sets of $m$ input variables (generating $\{\mathbf{x}_T | \vert T\vert = n-m\}$). The average confidence monotonically decreases when more input variables are masked. (3) The decreasing speed of the average confidence is polynomial. Detailed conditions are introduced in Appendix~\ref{appendix:condition}. These conditions are found to be commonly satisfied by most DNNs designed for classification tasks, including LLMs\cite{ren2024where, shen2023can}.}. output score $v(\mathbf{x}_S)$ on different masked samples $\mathbf{x}_S$. $\Gamma = \vert \{ S\subseteq N: I(S|\mathbf{x})\vert>\tau\}\vert$ denotes the number of salient interactions extracted from the input $\mathbf{x}$, where $\tau$ is a threshold for sailent interactions, then $\Gamma$ is on the order of magnitude of $\mathcal{O}(n^{\kappa}/\tau)$.
\end{theorem}

\begin{theorem}[Universal matching property, proved by~\cite{zhou2023explaining}]\label{thm:universalMatching}
Given an input sample $\mathbf{x}$, let us construct the following surrogate logical model $f(\cdot)$ to use AND-OR interactions extracted from the DNN $v(\cdot)$ on the sample $\mathbf{x}$ for inference. Then, the surrogate logical model $f(\cdot)$ can well mimic the DNN's output on $2^n$ masked samples.
\begin{align}
\forall T\subseteq N, f(\mathbf{x}_T) = &v(\mathbf{x}_T)\nonumber,\\
f(\mathbf{x}_T) = &v(\mathbf{x}_\emptyset) +\!\!\!\sum\nolimits_{S\subseteq N}I_{\text{\rm and}}(S|\mathbf{x}) \cdot \mathbbm{1}\left(\!\!\!\!\begin{array}{cc}
    {\scriptstyle \mathbf{x}_T\text{ \rm{triggers}}}\\
      {\scriptstyle \text{\rm {AND relation }}} S
\end{array}\!\!\!\!\right)
+\!\!\!\sum\nolimits_{S\subseteq N}I_{\text{\rm or}}(S|\mathbf{x}) \cdot \mathbbm{1}\left(\!\!\!\!\begin{array}{cc}
    {\scriptstyle \mathbf{x}_T\text{  \rm{triggers}}}\\
      {\scriptstyle \text{\rm {OR relation }}S}
\end{array}\!\!\!\!\right) \nonumber \\
= &v(\mathbf{x}_\emptyset) + \sum\nolimits_{0\neq S\subseteq T}I_{\text{\rm and}}(S|\mathbf{x})
+\sum\nolimits_{S\in \{S:S\cap T\neq \emptyset\}}I_{\text{\rm or}}(S|\mathbf{x})\nonumber \\
\approx  & v(\mathbf{x}_\emptyset) + \sum\nolimits_{S\in \Omega_{\text{\rm and}}: \emptyset \neq S\subseteq T}I_{\text{\rm and}}(S|\mathbf{x}) + \sum\nolimits_{S\in \Omega_{\text{\rm or}}: S\cap T \neq \emptyset}I_{\text{\rm or}}(S|\mathbf{x}),
\end{align}
where, $\Omega_{\text{\rm and}}\!\! =\!\!\{ S\!\subseteq \!N:|I_{\text{\rm and}}(S|\mathbf{x})|>\tau\}$, $\Omega_{\text{\rm or}} \!\!=\!\!\{S \subseteq N: |I_{\text{\rm or}}(S|\mathbf{x})|>\tau\}$ and $\tau$ is a small threshold.
\end{theorem}

\subsection{Context-agnostic memorization vs. in-context reasoning}\label{sec:mvsr}
Let us be given the prompt $\mathbf{x}=[\mathbf{x}_p^\top, \mathbf{x}_q^\top]^\top$, which consists of a premise $\mathbf{x}_p=\{$\textit{Caren works at school as a teacher.}$\}$, and a question $\mathbf{x}_q=\{$\textit{Emily is the colleague of Caren, Emily works as a}$\}$. Let us assume that the question can only be answered by incorporating the in-context reasoning with the premise $\mathbf{x}_p$. Thus, the answer generated without the premise $v(x_{n+1}|\mathbf{x}_q)$ is considered to exclusively use the context-agnostic memory, although the memorized knowledge may contain both meaningful knowledge and meaningless/over-fitted patterns. In comparison, given both the premise and question, the language generation $v(x_{n+1}|\mathbf{x}=[\mathbf{x}_p^\top, \mathbf{x}_q^\top]^\top)$ is believed to use both memorization and reasoning for inference.

Thus, in this paper, we hope to go beyond simply decomposing the output score $v(x_{n+1}|\mathbf{x}=[\mathbf{x}_p^\top, \mathbf{x}_q^\top]^\top)$ into the score component for context-agnostic memorization and the score component for in-context reasoning. Instead, in Section~\ref{sec:axiomaticsystem}, we will propose two axioms to formulate and extract tens/hundreds of explicit interactions corresponding to context-agnostic memorization effect and interactions corresponding to in-context reasoning effects. Specifically, we categorize the memorization effects into (1) the foundational memorization effect and (2) the chaotic memorization effect, and to categorize the reasoning effects into (1) the enhanced inference pattern, (2) the eliminated inference pattern, and (3) the reversed inference pattern.

\subsection{Axiomatic system}\label{sec:axiomaticsystem}

We hope to precisely define and quantify the effects of context-agnostic memorization and in-context reasoning in the LLM's inference logic. Specifically, given the input prompt $\mathbf{x}=[\mathbf{x}_p^\top, \mathbf{x}_q^\top]^\top$, let us extract a set of salient AND interactions $\Omega_{\text{\rm and}}(\mathbf{x})=\{S\subseteq N: \vert I_{\text{\rm and}}(S|\mathbf{x})\vert >\tau\}$ and a set of salient OR interactions $\Omega_{\text{\rm or}}(\mathbf{x}) =\{S\subseteq N: \vert I_{\text{\rm or}}(S|\mathbf{x})\vert >\tau\}$. Then, $I_{\text{\rm and}}(S|\mathbf{x}_q)$, $\Omega_{\text{\rm and}}(\mathbf{x}_q)$, $I_{\text{\rm or}}(S|\mathbf{x}_q)$ and $\Omega_{\text{\rm or}}(\mathbf{x}_q)$ correspond to terms extracted when we only input $\mathbf{x}_q$ into the LLM.

\begin{figure}[t]
  \begin{minipage}[c]{0.4\textwidth}
    \includegraphics[clip, trim=2cm 12.5cm 16cm 4cm, width=0.8\textwidth]{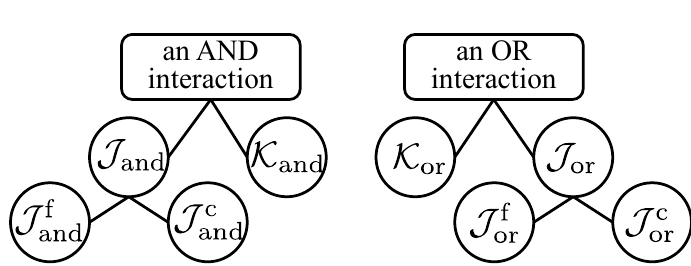}
  \end{minipage}
  \begin{minipage}[c]{0.6\textwidth}
    \caption{
       Illustration of the decomposition of an interaction into the foundational memorization interaction ($\mathcal{J}^{\text{f}}_{\text{and}}$, $\mathcal{J}^{\text{f}}_{\text{or}}$), the chaotic memorization interaction ($\mathcal{J}^{\text{c}}_{\text{and}}$, $\mathcal{J}^{\text{c}}_{\text{or}}$), and the in-context reasoning interaction ($\mathcal{K}_{\text{and}}$, $\mathcal{K}_{\text{or}}$).
    }  \label{fig:fig2}
  \end{minipage}
\end{figure}

In this way, given an input prompt $\mathbf{x}=[\mathbf{x}_p^\top, \mathbf{x}_q^\top]^\top$, the goal of this study is to define and decompose all AND-OR interactions ($I_{\text{\rm and}}(S|\mathbf{x})$, $I_{\text{\rm or}}(S|\mathbf{x})$) as the sum of context-agnostic memorization AND-OR interactions ($\mathcal{J}_{\text{\rm and}}(S|\mathbf{x})$, $\mathcal{J}_{\text{\rm or}}(S|\mathbf{x})$) and in-context reasoning AND-OR interactions ($\mathcal{K}_{\text{\rm and}}(S|\mathbf{x})$, $\mathcal{K}_{\text{\rm or}}(S|\mathbf{x})$), as shown in Figure~\ref{fig:fig2}. Thus, the core task in this subsection is to propose an axiomatic system to define and quantify $\mathcal{J}_{\text{\rm and}}(S|\mathbf{x})$, $\mathcal{J}_{\text{\rm or}}(S|\mathbf{x})$, $\mathcal{K}_{\text{\rm and}}(S|\mathbf{x})$ and $\mathcal{K}_{\text{\rm or}}(S|\mathbf{x})$. 
\begin{subequations}
    \begin{gather}
            I_{\text{\rm and}}(S|\mathbf{x}) \overset{\mathrm{decompose}}{=} \mathcal{J}_{\text{\rm and}}(S|\mathbf{x}) + \mathcal{K}_{\text{\rm and}}(S|\mathbf{x}), ~~~~~~ I_{\text{\rm or}}(S|\mathbf{x}) \overset{\mathrm{decompose}}{=} \mathcal{J}_{\text{\rm or}}(S|\mathbf{x}) + \mathcal{K}_{\text{\rm or}}(S|\mathbf{x}), \label{eq:IJK}\\
            \Rightarrow v(x_{n+1}|\mathbf{x}) 
            =  \!\!\!\!\underbrace{\!\!\!\!\!\sum_{S\in \Omega_{\text{\rm and}}(\mathbf{x})}\!\!\!\!\!\mathcal{J}_{\text{\rm and}}(S|\mathbf{x})  
            +\!\!\!\!\!\!\!\sum_{S\in \Omega_{\text{\rm or}}(\mathbf{x})}\!\!\!\!\!\mathcal{J}_{\text{\rm or}}(S|\mathbf{x})}_\text{inference effects based on memorization}
            +\!\!\!\! \underbrace{\!\!\!\!\!\sum_{S\in \Omega_{\text{\rm and}}(\mathbf{x})}\!\!\!\!\!\mathcal{K}_{\text{\rm and}}(S|\mathbf{x}) 
            + \!\!\!\!\!\!\sum_{S\in \Omega_{\text{\rm or}}(\mathbf{x})}\!\!\!\!\!\mathcal{K}_{\text{\rm or}}(S|\mathbf{x})} _\text{inference effects based on reasoning}+v(x_{n+1}|\mathbf{x}_\emptyset). \label{vJK}
    \end{gather}
\end{subequations}

\begin{axiom}[Condition dependence]\label{axiom1}
    In-context reasoning effects are supposed to exclusively affect the confidence score with the premise, \emph{i.e.}, if $\mathbf{x}=\mathbf{x}_q$, then $\mathcal{K}_{\text{\rm and}}(S|\mathbf{x}=\mathbf{x}_q)=0$ and $\mathcal{K}_{\text{\rm or}}(S|\mathbf{x}=\mathbf{x}_q)$=0.
\end{axiom}

The axiom requires that the interaction effect for in-context reasoning must depend on the contextual premise $\mathbf{x}_p$. \emph{I.e.}, if we remove the premise $\mathbf{x}_p$, then all in-context reasoning effects ($\mathcal{K}_{\text{\rm and}}(S|\mathbf{x}_q)$, $\mathcal{K}_{\text{\rm or}}(S|\mathbf{x}_q)$) are supposed to be eliminated, and the remaining inference effects correspond to context-agnostic memorization ($\mathcal{J}_{\text{\rm and}}(S|\mathbf{x}_q)$, $\mathcal{J}_{\text{\rm or}}(S|\mathbf{x}_q)$).

\textbf{Foundational memorization vs. chaotic memorization.} In addition, the memorization effects within $\mathcal{J}_{\text{\rm and}}(S|\mathbf{x})$ and $\mathcal{J}_{\text{\rm or}}(S|\mathbf{x})$ can be further divided into two types, \emph{i.e.}, decomposing $\mathcal{J}_{\text{\rm and}}(S|\mathbf{x}) = \mathcal{J}^{\text{\rm f}}_{\text{\rm and}}(S|\mathbf{x}) + \mathcal{J}^{\text{\rm c}}_{\text{\rm and}}(S|\mathbf{x})$ and $\mathcal{J}_{\text{\rm or}}(S|\mathbf{x}) = \mathcal{J}^{\text{\rm f}}_{\text{\rm or}}(S|\mathbf{x}) + \mathcal{J}^{\text{\rm c}}_{\text{\rm or}}(S|\mathbf{x})$. The first type of memorization is the foundational memorization, which represents the retention of common facts from the training corpus, denoted by $\mathcal{J}^{\text{\rm f}}_{\text{\rm and}}(S|\mathbf{x})$ and $\mathcal{J}^{\text{\rm f}}_{\text{\rm or}}(S|\mathbf{x})$. For example, given the input $\mathbf{x}=\{$\textit{The sun rises from the}$\}$, the LLM uses such memorization to predict the next word \textit{east}.

The second type of the memorization effect is chaotic memorization. The chaotic memorization effect is referred to as the logically chaotic patterns associated with the premise and the question. For example, let us be given two logically equivalent prompts $\mathbf{x}=\{$\textit{John is the father of Tina, Tina is the daughter of}$\}$ and $\mathbf{x}'=\{$\textit{\underline{John likes singing.} John is the father of Tina, Tina is the daughter of}$\}$. The only difference is \textit{John likes singing}, which is supposed not to affect the reasoning logic to predict the next word \textit{John}. Thus, the chaotic memorization effect represents the interaction effects that are affected by the logical-irrelevant textual changes in two logically equivalent prompts. In this way, we propose the following axiom.

\begin{axiom}[Variable independence]\label{axiom2}
Given a set of logically equivalent prompts, $\mathbb{X} = \{\mathbf{x}^1, \mathbf{x}^2,\dots, \mathbf{x}^k\}$, let $v(x_{n+1}|\mathbb{X}) = \mathbb{E}_i[v(x_{n+1}|\mathbf{x}^i) ]$ represent the average confidence score. Then, the 
 change of confidence scores \emph{w.r.t.} $v(x_{n+1}|\mathbf{x})$ and $v(x_{n+1}|\mathbb{X})$ are supposed to be exclusively caused by chaotic memorization effects as follows,
\begin{equation}\label{eq:axiom2}
        v(x_{n+1}|\mathbf{x}) - v(x_{n+1}|\mathbb{X})=\sum\nolimits_{S\in \Omega_{\text{\rm and}}(\mathbf{x})}\mathcal{J}^{\text{\rm c}}_{\text{\rm and}}(S|\mathbf{x})
        +\sum\nolimits_{S\in \Omega_{\text{\rm or}}(\mathbf{x})}\mathcal{J}^{\text{\rm c}}_{\text{\rm or}}(S|\mathbf{x}). 
\end{equation}
All logically equivalent prompts in $\mathbb{X}$ have the same foundational memorization effects and the in-context reasoning effects, \emph{i.e.}, 
$\forall \mathbf{x}^i, \mathbf{x}^j$ in $\mathbb{X}, \mathcal{J}^{\text{\rm f}}_{\text{\rm and}}(S|\mathbf{x}^i) = \mathcal{J}^{\text{\rm f}}_{\text{\rm and}}(S|\mathbf{x}^j), \mathcal{K}_{\text{\rm and}}(S|\mathbf{x}^i) = \mathcal{K}_{\text{\rm and}}(S|\mathbf{x}^j)$ and 
$\mathcal{J}^{\text{\rm f}}_{\text{\rm or}}(S|\mathbf{x}^i) = \mathcal{J}^{\text{\rm f}}_{\text{\rm or}}(S|\mathbf{x}^j), \mathcal{K}_{\text{\rm or}}(S|\mathbf{x}^i) = \mathcal{K}_{\text{\rm or}}(S|\mathbf{x}^j)$.
\end{axiom}

In LLMs, these chaotic memorization effects often constitute a relatively small proportion of all effects. In order to realize the axiom of variable independence, we construct the following three types of logically equivalent prompts. Let us take the premise $\mathbf{x}_p=\{$\textit{Caren works at school as a teacher.}$\}$, and the question $\mathbf{x}_q=\{$\textit{Emily is the colleague of Caren, Emily works as a}$\}$ as an example. The corresponding logically equivalent prompts are shown in both Table~\ref{tab:logical_equivalent} and Appendix~\ref{appendix:prompts}.

$\bullet$ \textit{Removing and adding irrelevant background.} The reasoning logic should be insensitive to irrelevant background in the premise $\mathbf{x}_p$, such as additional noise details. For instance, let $\mathbf{x}_p'=\{$\textit{Caren works as a teacher and Tom works as a doctor.}$\}$. The presence of the information \textit{``and Tom works as a doctor''} should not impact the reasoning process. Therefore, the two prompts $\mathbf{x}=\{$\textit{Caren works as a teacher. Emily is the colleague of Caren, Emily works as a}$\}$ and $\mathbf{x}'=\{$\textit{Caren works as a teacher \underline{and Tom works as a doctor}. Emily is the colleague of Caren, Emily works as a}$\}$ are logically equivalent.

$\bullet$ \textit{Semantic paraphrasing.} This involves rephrasing the premise $\mathbf{x}_p$ using synonymous or semantically equivalent expressions. For example, $\mathbf{x}''=\{$\textit{\underline{Caren teaches students at school}. Emily is the colleague of Caren, Emily works as a}$\}$ and $\mathbf{x}$ are logically equivalent. 

$\bullet$ \textit{Renaming.} The name of entities involved in the reasoning process does not directly affect the reasoning logic. Therefore, when we simply rename these entities in both the premise $\mathbf{x}_p$ and the question $\mathbf{x}_q$, \emph{e.g.}, replacing \textit{``Caren''} in $\mathbf{x}$ with any other names such as \textit{``Tina''} to obtain $\mathbf{x}''' = \{$\textit{\underline{Tina} works at school as a teacher. Emily is the colleague of \underline{Tina}, Emily works as a}$\}$, the two prompts are logically equivalent.

\begin{table}[t]
    \centering
 \resizebox{\linewidth}{!}{\footnotesize
    \begin{tabular}{l|l}
    \toprule
        Type &  Logically equivalent prompt \\
    \midrule
    \midrule
        Original prompt & \textcolor{orange}{Caren works as a teacher.} \textcolor{teal}{Emily is the colleague of Caren, Emily works as a}\\
        \midrule
         \multirow{2}{*}{\makecell{Removing and adding \\irrelevant background}} & a. \textcolor{orange}{Caren works as a teacher \textbf{and Tom works as a doctor}.} \textcolor{teal}{Emily is the colleague of Caren, Emily works as a}\\
         & b. \textcolor{orange}{Caren works as a teacher \textbf{and she likes singing}}. \textcolor{teal}{Emily is the colleague of Caren, Emily works as a} \\
         \midrule
         \multirow{2}{*}{Semantic paraphrasing} & a. \textcolor{orange}{\textbf{Caren teaches students at school}.} \textcolor{teal}{Emily is the colleague of Caren, Emily works as a} \\
         &b. \textcolor{orange}{\textbf{The school hired Caren as a teacher}.} \textcolor{teal}{Emily is the colleague of Caren, Emily works as a} \\
         \midrule
         \multirow{2}{*}{Renaming} &a. \textcolor{orange}{\textbf{Tina} works at school as a teacher.} \textcolor{teal}{Emily is the colleague of \textbf{Tina}, Emily works as a} \\
         &b. \textcolor{orange}{\textbf{Anna} works at school as a teacher.} \textcolor{teal}{Emily is the colleague of \textbf{Anna}, Emily works as a} \\
    \bottomrule 
    \end{tabular}}
    \vspace{1pt}
    \caption{Examples of three types of logically equivalent prompts.}
    \label{tab:logical_equivalent}
\end{table}

\subsection{Fine-grained analysis of memorization effects and in-context reasoning effects}\label{sec:fine-grained}
Since the above two axioms have defined properties that the context-agnostic memorization effect and the in-context reasoning effect are supposed to satisfy, in this subsection, we will follow the two axioms to quantify the exact memorization effects and reasoning effects. Then, we further propose a series of new metrics to categorize these effects into more fine-grained types.

\textbf{Quantification of the memorization effect ($\mathcal{J}_{\text{\rm and}}(S|\mathbf{x})$, $\mathcal{J}_{\text{\rm or}}(S|\mathbf{x})$) and the reasoning effect ($\mathcal{K}_{\text{\rm and}}(S|\mathbf{x})$, $\mathcal{K}_{\text{\rm or}}(S|\mathbf{x})$).} According to Axioms~\ref{axiom1} and~\ref{axiom2}, we can quantify the foundational memorization effect ($\mathcal{J}^{\text{\rm f}}_{\text{\rm and}}(S|\mathbf{x})$, $\mathcal{J}^{\text{\rm f}}_{\text{\rm or}}(S|\mathbf{x})$), the chaotic memorization effect ($\mathcal{J}^{\text{\rm c}}_{\text{\rm and}}(S|\mathbf{x})$, $\mathcal{J}^\text{\rm c}_{\text{\rm or}}(S|\mathbf{x})$) and the reasoning effect ($\mathcal{K}_{\text{\rm and}}(S|\mathbf{x})$, $\mathcal{K}_{\text{\rm or}}(S|\mathbf{x})$) as follows.
\begin{subequations}
    \begin{align}
        \mathcal{J}^{\text{\rm f}}_{\text{\rm and}}(S|\mathbf{x}) &= I_{\text{\rm and}}(S|\mathbf{x}_q), &\quad \mathcal{J}^{\text{\rm f}}_{\text{\rm or}}(S|\mathbf{x}) &= I_{\text{\rm or}}(S|\mathbf{x}_q), \label{eq:J_f}\\
        \mathcal{J}^{\text{\rm c}}_{\text{\rm and}}(S|\mathbf{x}) &= I_{\text{\rm and}}(S|\mathbf{x}) - I_{\text{\rm and}}(S|\mathbb{X}), &
        \quad \mathcal{J}^{\text{\rm c}}_{\text{\rm or}}(S|\mathbf{x}) &= I_{\text{\rm or}}(S|\mathbf{x})-I_{\text{\rm or}}(S|\mathbb{X}), \label{eq:J_c}\\
        \mathcal{K}_{\text{\rm and}}(S|\mathbf{x}) &= I_{\text{\rm and}}(S|\mathbb{X}) - I_{\text{\rm and}}(S|\mathbf{x}_q), &
        \quad \mathcal{K}_{\text{\rm or}}(S|\mathbf{x}) &= I_{\text{\rm or}}(S|\mathbb{X}) - I_{\text{\rm or}}(S|\mathbf{x}_q), \label{eq:K}
    \end{align}
\end{subequations}
where $I_{\text{\rm and}}(S|\mathbb{X})$ and $I_{\text{\rm or}}(S|\mathbb{X})$ is computed based on $v(x_{n+1}|\mathbb{X}_T) = \mathbb{E}_i[v(x_{n+1}|\mathbf{x}^i_T) ]$ to represent the average AND-OR interactions. Here, $\mathbb{X}_T=\{\mathbf{x}^1_T, \dots, \mathbf{x}^k_T\}$ represents a set of $k$ masked logically equivalent prompts. Notably, the $n$ input words annotated in the original prompt should also be found and annotated in the $k-1$ logically equivalent prompts to enable the computation of $v(x_{n+1} |\mathbf{x}_S)$\footnote{For more detailed information, please refer to Appendix~\ref{appendix:experimentaldetails}}. In this way, the confidence score $ v(x_{n+1}|\mathbf{x})$ can be represented as the sum of the above three types of interactions.
{\small
\begin{equation}\label{eq:eq7}
    v(x_{n+1}|\mathbf{x})\!= \!\!\!\!\!\!\!\!\sum_{S\in \Omega_{\text{\rm and}}(\mathbf{x})}\!\!\!\!\!\!\!\left(\mathcal{J}^{\text{\rm f}}_{\text{\rm and}}(S|\mathbf{x})\!+\!\mathcal{J}^{\text{\rm c}}_{\text{\rm and}}(S|\mathbf{x})\!+\!\mathcal{K}_{\text{\rm and}}(S|\mathbf{x})\right) 
    \!+\!\!\!\!\!\!\!\sum_{S\in \Omega_{\text{\rm or}}(\mathbf{x})}\!\!\!\!\!\!\!\left(\mathcal{J}^{\text{\rm f}}_{\text{\rm or}}(S|\mathbf{x}) \!+\!\mathcal{J}^{\text{\rm c}}_{\text{\rm or}}(S|\mathbf{x}) \!+\! \mathcal{K}_{\text{\rm or}}(S|\mathbf{x})\right) + v(x_{n+1}|\mathbf{x}_\emptyset).
\end{equation}}

\textbf{Theoretical guarantee of the faithfulness: The above decomposed effects still satisfy the sparsity property and the universal matching property.} $\mathcal{J}^{\text{f}}_{\text{\rm and}}(S|\mathbf{x})$, $\mathcal{J}^{\text{f}}_{\text{\rm or}}(S|\mathbf{x})$, $\mathcal{J}^{\text{c}}_{\text{\rm and}}(S|\mathbf{x})$, $\mathcal{J}^{\text{c}}_{\text{\rm or}}(S|\mathbf{x})$, $\mathcal{K}_{\text{\rm and}}(S|\mathbf{x})$ and $\mathcal{K}_{\text{\rm or}}(S|\mathbf{x})$ are all sparse due to the sparsity of $I_{\text{\rm and}}(S|\mathbf{x})$ and $I_{\text{\rm or}}(S|\mathbf{x})$ (see Theorem~\ref{thm:upperbound}). Moreover, Equation~(\ref{eq:eq7}) shows that the LLM's output on any randomly masked sample $\mathbf{x}_T$ can always be written as the sum of the decomposed effects, \emph{i.e.}, $\forall T\subseteq N$, $v(x_{n+1}|\hat{\mathbf{x}}_T) =\sum_{S\in \Omega_{\text{\rm and}}(\hat{\mathbf{x}}):\emptyset\neq S\subseteq T}\left(\mathcal{J}^{\text{\rm f}}_{\text{\rm and}}(S|\hat{\mathbf{x}})\!+\!\mathcal{J}^{\text{\rm c}}_{\text{\rm and}}(S|\hat{\mathbf{x}})\!+\!\mathcal{K}_{\text{\rm and}}(S|\hat{\mathbf{x}})\right) 
+\sum_{S\in \Omega_{\text{\rm or}}(\hat{\mathbf{x}}):S\cap T\neq \emptyset}\left(\mathcal{J}^{\text{\rm f}}_{\text{\rm or}}(S|\hat{\mathbf{x}}) \!+\!\mathcal{J}^{\text{\rm c}}_{\text{\rm or}}(S|\hat{\mathbf{x}}) \!+\! \mathcal{K}_{\text{\rm or}}(S|\hat{\mathbf{x}})\right) + v(x_{n+1}|\hat{\mathbf{x}}_\emptyset).$ Please see Appendix~\ref{appendix:universal matching prove} for proof. The sparsity property and the universal property mathematically ensure the faithfulness of using the decomposed effects to explain the language generation of an LLM. It is because we can use a small number of salient interactions to predict every fine-grained changes of the LLM on $2^n$ differently masked samples.

\textbf{Visualization of the distribution of memorization effects and reasoning effects over different orders.} Here, we use the \textit{order} of the interaction to represent the number of input words in $S$, \emph{i.e.}, $\textit{order}(S) = |S|$.~\citet{zhou2024explaining} have discovered that high-order interactions have lower generalization power than low-order interactions. For each order $m$, we quantify the strength of all positive salient foundational memorization interactions of the $m$-th order $\mathcal{E}^{m, \text{pos}}_{{\text{\rm f}}} = \sum_{\text{type} \in \{\text{and, or}\}}\sum_{S\in\Omega_{\text{\rm type}}:|S|=m}\max(0, \mathcal{J}_{\text{\rm type}}^{\text{\rm f}}(S|\mathbf{x}))$ and the strength of all negative salient foundational memorization interactions of the $m$-th order $\mathcal{E}^{m, \text{neg}}_{{\text{\rm f}}} = -\sum_{\text{type} \in \{\text{and, or}\}}\sum_{S\in\Omega_{\text{\rm type}}:|S|=m}\min(0, \mathcal{J}_{\text{\rm type}}^{\text{\rm f}}(S|\mathbf{x}))$. Similarly, we can obtain $\mathcal{E}^{m, \text{pos}}_{\text{\rm c}}$ and $\mathcal{E}^{m, \text{neg}}_{\text{\rm c}}$ for chaotic memorization interactions, and $\mathcal{E}^{m, \text{pos}}_{\mathcal{K}}$ and $\mathcal{E}^{m, \text{neg}}_{\mathcal{K}}$ for in-context reasoning interactions. Figure~\ref{fig:fig3} visualizes the foundational memorization effects, chaotic memorization effects and reasoning effects over different orders. 

In addition, we use the following metric to quantify the ratio of the in-context reasoning effect $\rho_{\text{r}}$ and the ratio of the chaotic memorization effect  $\rho_{\text{c}}$ to all effects $\mathcal{E}_{\text {all}}$. 
\begin{equation}
    \rho_{\text{r}} = \frac{\sum_m\mathcal{E}^{m, \text{pos}}_{\mathcal{K}} +   \mathcal{E}^{m, \text{neg}}_{\mathcal{K}}}{\mathcal{E}_{\text{all}}}, \quad
    \rho_{\text{c}} =  \frac{\sum_m\mathcal{E}^{m, \text{pos}}_{\text{\rm c}} + \mathcal{E}^{m, \text{neg}}_{\text{\rm c}}}{\mathcal{E}_{\text{all}}},
\end{equation}
where $\mathcal{E}_{\text{all}} = \sum_{m}\mathcal{E}^{m, \text{pos}}_{\text{\rm f}}  +  \mathcal{E}^{m, \text{neg}}_{\text{\rm f}}  +  \mathcal{E}^{m, \text{pos}}_{\text{\rm c}}  +  \mathcal{E}^{m, \text{neg}}_{\text{\rm c}} +  \mathcal{E}^{m, \text{pos}}_{\mathcal{K}}  + \mathcal{E}^{m, \text{neg}}_{\mathcal{K}}$.
  \begin{table}[t]
  \begin{minipage}[c]{0.57\textwidth}
  \centering
  {\tiny
    \begin{tabular}{c ccc}
    \toprule
    &OPT-1.3B & LLaMA-7B & GPT-3.5-Turbo \\
    \midrule
    $\rho_{\text{r}} \uparrow$ & $39.12\%$  &  $42.82\%$& $\mathbf{45.24\%}$ \\
    \midrule
    $\rho_{\text{c}} \downarrow$ & $\mathbf{7.37}\%$  &  $7.57\%$& $7.81\%$ \\
    \bottomrule 
    \end{tabular}}
      \end{minipage}
      \begin{minipage}[c]{0.4\textwidth}
          \caption{The ratio $\rho_{\text{r}}$ of in-context reasoning effects and the ratio $\rho_{\text{c}}$ of chaotic memorization effects.}
    \label{tab:ratio}
      \end{minipage}
\end{table}   

\begin{figure}[t]
  \begin{minipage}[c]{0.67\textwidth}
    \includegraphics[clip, trim=0cm 0cm 0cm 0cm, width=0.95\textwidth]{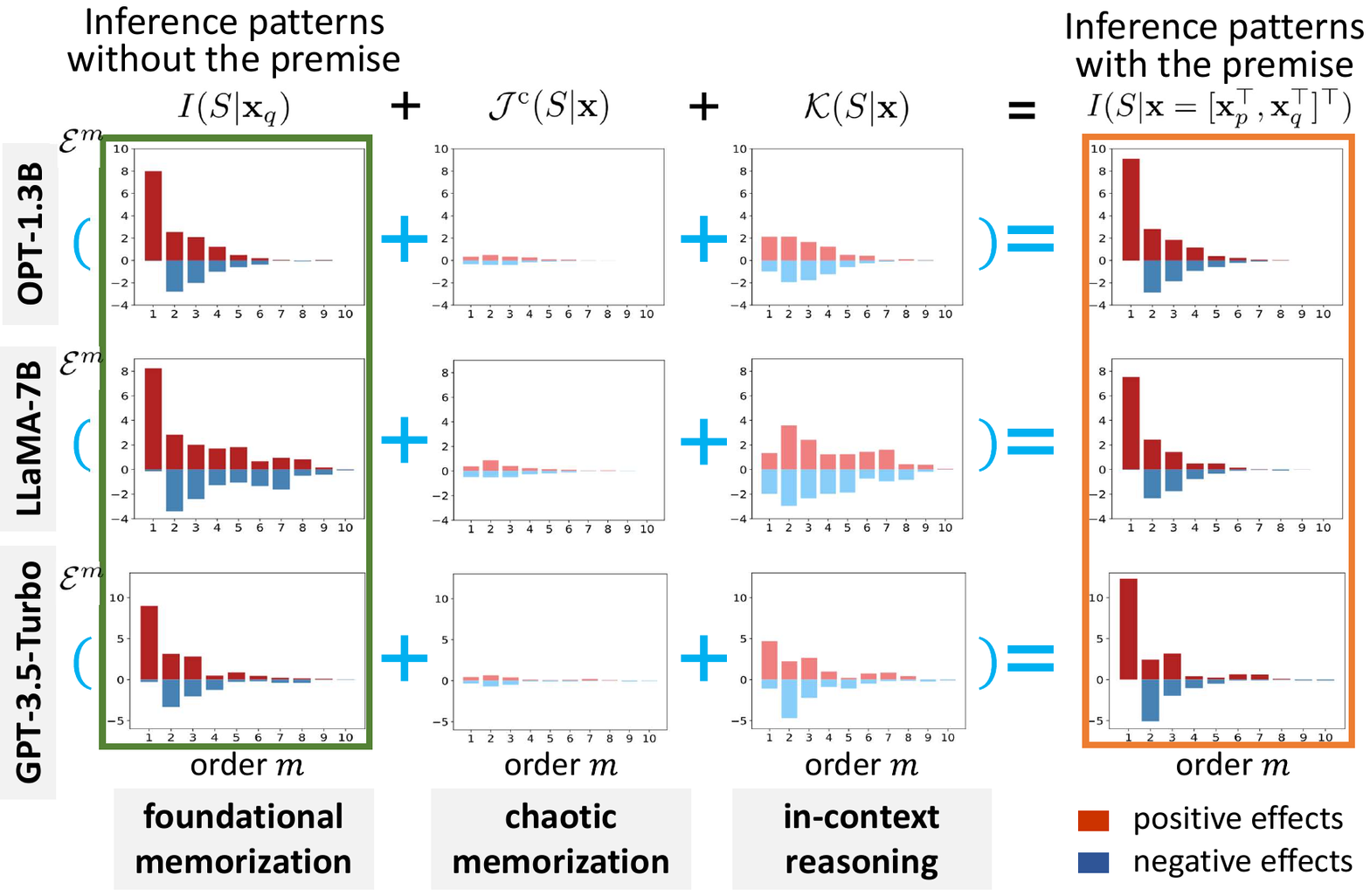}
  \end{minipage}
  \begin{minipage}[c]{0.3\textwidth}
    \caption{
       Distribution of the strength of foundational memorization interactions ($\mathcal{J}^{\text{f}}_{\text{and}}(S|\mathbf{x}), \mathcal{J}^{\text{f}}_{\text{or}}(S|\mathbf{x})$), chaotic memorization interactions ($\mathcal{J}^{\text{c}}_{\text{and}}(S|\mathbf{x}), \mathcal{J}^{\text{c}}_{\text{or}}(S|\mathbf{x})$) and in-context reasoning interactions ($\mathcal{K}_{\text{and}}(S|\mathbf{x}), \mathcal{K}_{\text{or}}(S|\mathbf{x})$) by order $m$. The results are averaged over all samples. Please see experimental settings in Section~\ref{sec:fine-grained} for details, and Appendix~\ref{appendix:result} for results on the individual sample.
    }  \label{fig:fig3}
  \end{minipage}
\end{figure}

Table~\ref{tab:ratio} shows the in-context reasoning effect ratio $\rho_{\text{r}}$ and the chaotic memorization effect ratio $\rho_{\text{c}}$ tested on three LLMs: OPT-1.3b, LLaMA-7B, and GPT-3.5-Turbo\footnote{Please see the experimental settings in the end of Section~\ref{sec:fine-grained}.}. \textbf{We observed that all three models encoded substantial in-context reasoning effects in inference and relatively weak chaotic memorization effects. The low ratio of chaotic memorization effects indicated the good representation quality of the LLM.} Among the three models, GPT-3.5-Turbo demonstrated the highest in-context reasoning effect ratio.

\textbf{Memorization effects contained interactions of varying orders.} As shown in Figure~\ref{fig:fig3}, when we only input the question $\mathbf{x}_q$ into the LLM, \emph{i.e.}, only using foundational memorization effects in the inference, the LLM encodes interactions with all orders.

\textbf{Fine-grained analysis of in-context reasoning effect.} Then, in this subsection, we further categorize the reasoning effect ($\mathcal{K}_{\text{\rm and}}(S|\mathbf{x})$, $\mathcal{K}_{\text{\rm or}}(S|\mathbf{x})$) into the following three types.

$\bullet$ \textit{The enhanced inference patterns.} The enhanced inference patterns are referred to as in-context reasoning effects which are enhanced by the premise $\mathbf{x}_p$, \emph{i.e.}, when we add the premise $\mathbf{x}_p$ into the prompt, the interaction strength $\vert I_{\text{and}}(S|\mathbf{x})\vert=\vert\mathcal{J}_{\text{and}}(S|\mathbf{x})+\mathcal{K}_{\text{and}} (S|\mathbf{x})\vert$ is enhanced by the reasoning effect $\mathcal{K}_{\text{and}} (S|\mathbf{x})$. Given all AND interactions in $\Omega_{\text {and}}$, the enhanced inference patterns are identified as follows.
\begin{equation}
        \Omega_{\text{\rm and}}^{\text{\rm enhanced}}=\{S\in \Omega_{\text{\rm and}}: \mathcal{J}_{\text{\rm and}}(S|\mathbf{x})\mathcal{K}_{\text{\rm and}}(S|\mathbf{x})>0\}.
\end{equation}
$\bullet$ \textit{The eliminated inference patterns.}  The eliminated inference patterns are in-context reasoning effects which are eliminated by the premise $\mathbf{x}_p$,  \emph{i.e.}, when we add the premise $\mathbf{x}_p$ into the prompt, the strength of the interaction is reduced by the reasoning effect $\mathcal{K} (S|\mathbf{x})$. Given all AND interactions in $\Omega_{\text {and}}$, the eliminated inference patterns are defined as follows.
\begin{equation}
    \Omega_{\text{\rm and}}^{\text{\rm eliminated}}=\{S\in \Omega_{\text{\rm and}}:\mathcal{J}_{\text{\rm and}}(S|\mathbf{x})\mathcal{K}_{\text{\rm and}}(S|\mathbf{x})<0,~ \vert\mathcal{J}_{\text{\rm and}}(S|\mathbf{x})\vert \geq \vert\mathcal{K}_{\text{\rm and}}(S|\mathbf{x}) \vert\}. 
\end{equation}
$\bullet$ \textit{The reversed inference patterns.} The interactions extracted from the LLM can either positively or negatively impact the confidence score. The reversed inference patters are those in-context reasoning effects which change sign of the final interaction $I_{\text{and}}(S|\mathbf{x})$. Given all AND interactions in $\Omega_{\text {and}}$, the reversed inference patterns are identified as follows. 
\begin{equation}
        \Omega_{\text{\rm and}}^{\text{\rm reversed}} =\{S\in \Omega_{\text{\rm and}}:\mathcal{J}_{\text{\rm and}}(S|\mathbf{x})\mathcal{K}_{\text{\rm and}}(S|\mathbf{x})<0,~ \vert\mathcal{J}_{\text{\rm and}}(S|\mathbf{x})\vert < \vert\mathcal{K}_{\text{\rm and}}(S|\mathbf{x}) \vert\}.
\end{equation}
In this way, we have $\Omega_{\text{\rm and}}^{\text{\rm enhanced}}\cup \Omega_{\text{\rm and}}^{\text{\rm eliminated}} \cup \Omega_{\text{\rm and}}^{\text{\rm reversed}} = \Omega_{\text{and}}$. 

\begin{figure}[t]
    \centering
    \includegraphics[clip, trim=0cm 6cm 0cm 0cm, width=0.9\textwidth]{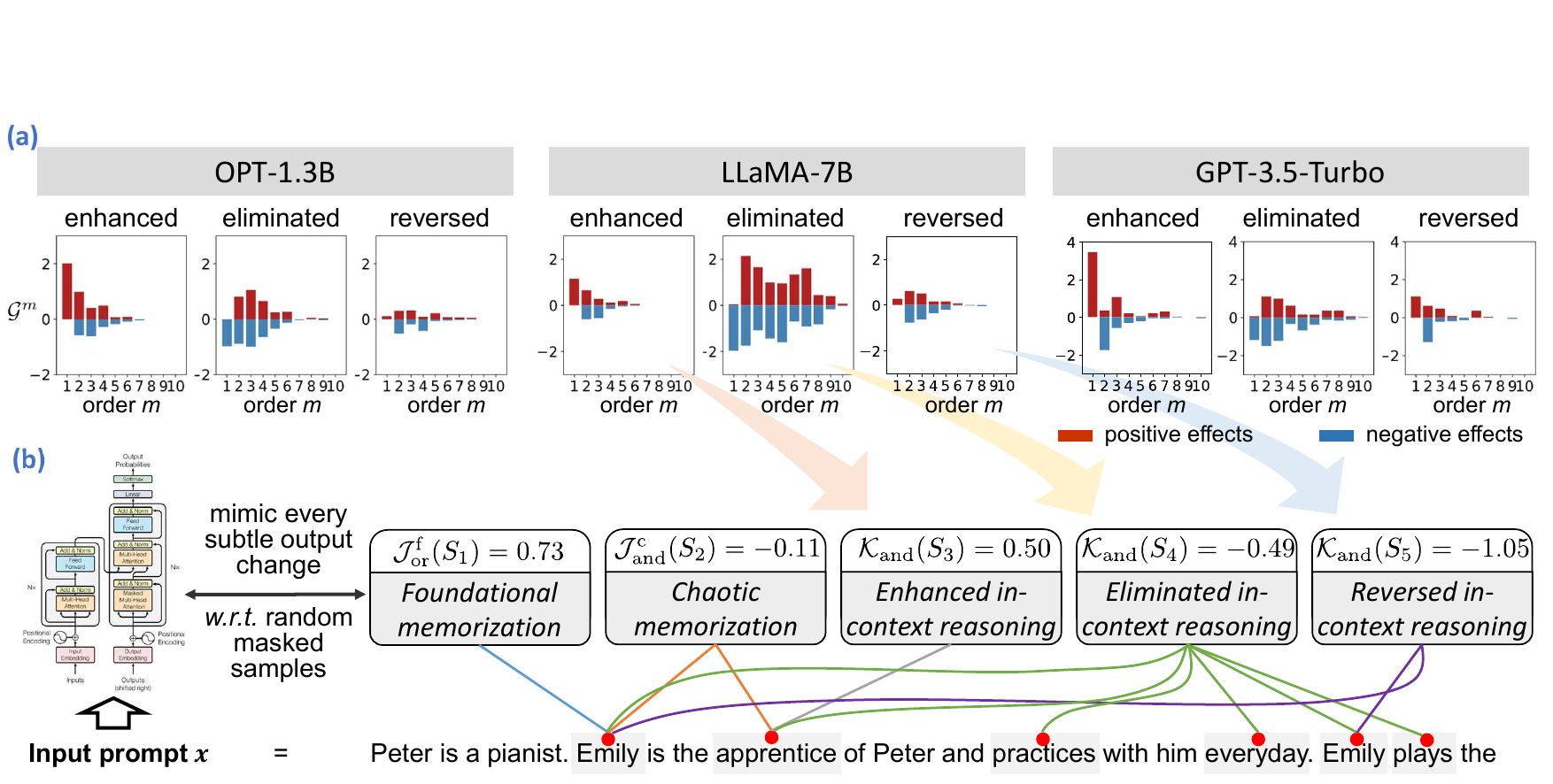}
    \caption{(a) Distribution of the strength of three types of in-context reasoning effects over different orders $m$. The results are averaged over all samples. (b) Visualization of some in-context reasoning effects and memorization effects on a single sample.}
    \label{fig:fig4}
\end{figure}

Similarly, we can identify $\Omega_{\text{\rm or}}^{\text{\rm enhanced}}$, $\Omega_{\text{\rm or}}^{\text{\rm eliminated}}$ and $\Omega_{\text{\rm or}}^{\text{\rm reversed}}$ for OR interactions. Figure~\ref{fig:fig4} visualizes the distribution of the above three types of reasoning effects over different orders. For each order $m$, we quantify the strength of the reasoning effect of all positive enhanced inference patterns, \emph{i.e.}, $\mathcal{G}^{m, {\text{\rm pos}}}_\text{enhanced} = \sum_{\text{type}\in \{\text{and, or}\}}\sum_{S\in \Omega_\text{type}^\text{enhanced}:|S|=m}\max(0, \mathcal{K}_{\text{type}}^\text{f}(S|\mathbf{x}))$ and $\mathcal{G}^{m, {\text{\rm neg}}}_\text{enhanced} = -\sum_{\text{type}\in \{\text{and, or}\}}\sum_{S\in \Omega_\text{type}^\text{enhanced}:|S|=m}\min(0, \mathcal{K}_{\text{type}}^\text{f}(S|\mathbf{x})) $. Similarly, we can obtain $\mathcal{G}^{m, {\text{\rm pos}}}_\text{eliminated}$ and $\mathcal{G}^{m, {\text{\rm neg}}}_\text{eliminated}$ for the eliminated inference patterns, $\mathcal{G}^{m, {\text{\rm pos}}}_\text{flipped}$ and $\mathcal{G}^{m, {\text{\rm neg}}}_\text{flipped}$ for the flipped inference patterns.

\textbf{Conclusion: in-context reasoning effects tend to eliminate high-order interactions and amplify low-order interactions in memorization effects.} As shown in Figure~\ref{fig:fig4}, when we added the premise to the prompt, reasoning effects mainly eliminated some memorization effects, especially those corresponding to high-order interactions. We can consider that the LLM uses the premise to filter out irrelevant memorization effects. Meanwhile, the reasoning effects also enhanced/reversed a small number of low-order interactions to refine the inference logic. Because high-order interactions have been found less generalizable than low-order interactions~\cite{zhou2024explaining}, the addition of the premise makes the inference much clearer and more faithful.

\begin{figure}[t]
    \centering
    \includegraphics[clip, trim=3cm 9.5cm 3cm 2.8cm, width=0.9\textwidth]{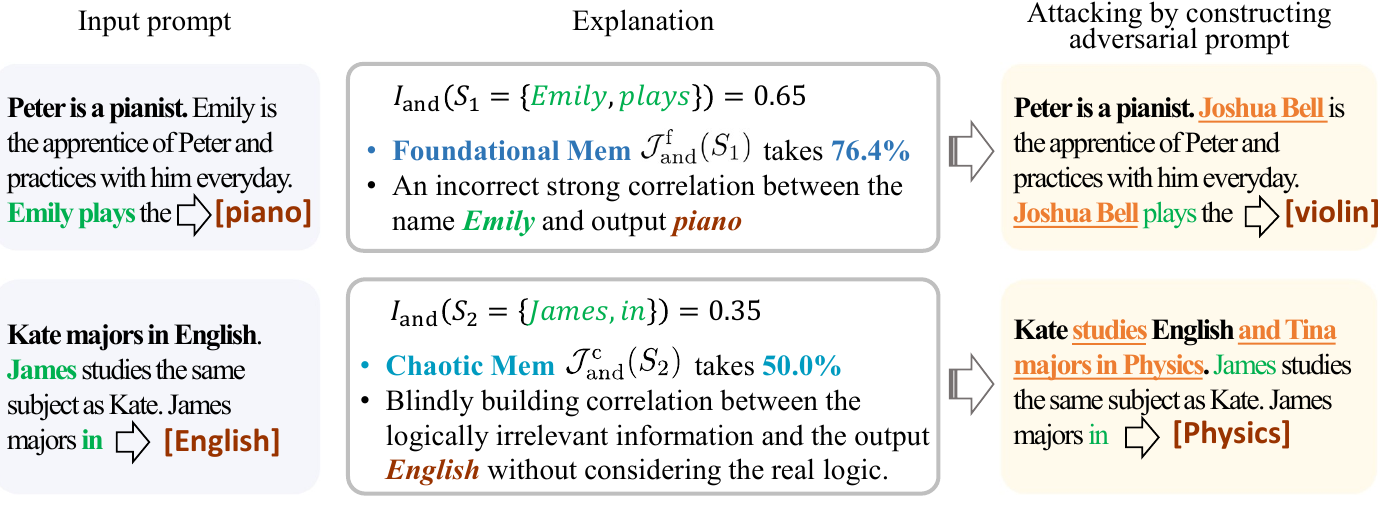}
    \caption{Interactions illustrate problematic inference patterns used by the LLM. Then, we construct adversarial prompts based on the interaction.}
    \label{fig:fig5}
\end{figure}

\textbf{Experimental settings.} We tested on some samples for a question-answer (QA) task. For each QA input sample, it contained a premise $\mathbf{x}_p$ and a question $\mathbf{x}_q$. We selected $n$ words\textsuperscript{\ref{footnote:choosewords}} in the question for each sample. For every sample, we constructed 10 logically equivalent prompts (see Table~\ref{tab:logical_equivalent}, each containing the same $n$ input words. For a detailed list of all examples, please refer to Appendix~\ref{appendix:prompts}.
We conducted experiments on three LLMs: OPT-1.3B \cite{zhang2022opt}, LLaMA-7B \cite{touvron2023llama}, and GPT-3.5-Turbo \cite{brown2020language}. We took the prompt in each QA sample, and let the LLM predict the next word $x_{n+1}$. 
Then, we calculated the confidence score $v(x_{n+1}|\mathbf{x})$. $x_{n+1}$ was determined as the real answer corresponding to the input prompt. 

\section{Conclusion and discussions on the practical value}\label{sec:conclusion}
In this paper, we have proposed an axiomatic system to rigorously define and quantify the context-agnostic memorization effects and the in-context reasoning effects used by the LLM for inference. We have further categorized the memorization effects into foundational memorization effects and chaotic memorization effects, and have classified the in-context reasoning effects into the enhanced inference patterns, the eliminated inference patterns and the reversed inference patterns. The proven sparsity and universal matching property of decomposed effects ensure the faithfulness of taking these effects as primitive inference patterns used by the LLM.

$\bullet$ \textbf{Identifying seemingly correct yet actually incorrect inference logic.} If we only evaluate the LLM based on its outputs, the superior performance of the LLM has led to a misconception that the LLM has already modeled reliable knowledge or conducted inference based on correct logic. However, our method first has enabled us to quantify the exact inference patterns of different types, and experiments have shown that many inference patterns represent incorrect logic.

$\bullet$ \textbf{Enabling semantic debugging.} Unlike traditional engineering explanation techniques~\cite{tenney2020language}, our method with theoretical guarantees (\emph{i.e.}, the universal-matching property) can guide the construction of adversarial prompts to attack LLMs. For instance, as shown in Figure~\ref{fig:fig5}, the interaction between $\{$\textit{Emily, plays}$\}$ makes significant contribution on the network output. However, its interaction effect primarily attributed to the foundational memorization effect ($76\%$). Therefore, the LLM mistakenly encoded a strong prior that associated the output \text{``piano''} with the interaction $\{$\textit{Emily, plays}$\}$. We constructed an adversarial prompt by replacing \textit{``Emily''} to \textit{``Joshua Bell''} to attack the LLM. In this way, we observed the LLM outputting \textit{``violin''} instead of \textit{``piano''}. 

We conducted a mathematical examination of the faithfulness of our interaction explanations (see Appendices~\ref{appendix:universal matching prove} and~\ref{appendix:faithfulness}). \textbf{In summary, our theoretical framework enables the examination of the correctness of detailed inference patterns encoded by LLM, which goes beyond the traditional evaluation of LLM's outputs}.

\newpage
\bibliographystyle{plainnat}
\bibliography{reference.bib}

\begin{thebibliography}{30}
\providecommand{\natexlab}[1]{#1}
\providecommand{\url}[1]{\texttt{#1}}
\expandafter\ifx\csname urlstyle\endcsname\relax
  \providecommand{\doi}[1]{doi: #1}\else
  \providecommand{\doi}{doi: \begingroup \urlstyle{rm}\Url}\fi

\bibitem[Amini et~al.(2019)Amini, Gabriel, Lin, Koncel-Kedziorski, Choi, and Hajishirzi]{amini2019mathqa}
Aida Amini, Saadia Gabriel, Peter Lin, Rik Koncel-Kedziorski, Yejin Choi, and Hannaneh Hajishirzi.
\newblock Mathqa: Towards interpretable math word problem solving with operation-based formalisms.
\newblock \emph{arXiv preprint arXiv:1905.13319}, 2019.

\bibitem[Brown et~al.(2020)Brown, Mann, Ryder, Subbiah, Kaplan, Dhariwal, Neelakantan, Shyam, Sastry, Askell, et~al.]{brown2020language}
Tom Brown, Benjamin Mann, Nick Ryder, Melanie Subbiah, Jared~D Kaplan, Prafulla Dhariwal, Arvind Neelakantan, Pranav Shyam, Girish Sastry, Amanda Askell, et~al.
\newblock Language models are few-shot learners.
\newblock \emph{Advances in neural information processing systems}, 33:\penalty0 1877--1901, 2020.

\bibitem[Carlini et~al.(2019)Carlini, Liu, Erlingsson, Kos, and Song]{carlini2019secret}
Nicholas Carlini, Chang Liu, {\'U}lfar Erlingsson, Jernej Kos, and Dawn Song.
\newblock The secret sharer: Evaluating and testing unintended memorization in neural networks.
\newblock In \emph{28th USENIX security symposium (USENIX security 19)}, pages 267--284, 2019.

\bibitem[Carlini et~al.(2021)Carlini, Tramer, Wallace, Jagielski, Herbert-Voss, Lee, Roberts, Brown, Song, Erlingsson, et~al.]{carlini2021extracting}
Nicholas Carlini, Florian Tramer, Eric Wallace, Matthew Jagielski, Ariel Herbert-Voss, Katherine Lee, Adam Roberts, Tom Brown, Dawn Song, Ulfar Erlingsson, et~al.
\newblock Extracting training data from large language models.
\newblock In \emph{30th USENIX Security Symposium (USENIX Security 21)}, pages 2633--2650, 2021.

\bibitem[Carlini et~al.(2022)Carlini, Ippolito, Jagielski, Lee, Tramer, and Zhang]{carlini2022quantifying}
Nicholas Carlini, Daphne Ippolito, Matthew Jagielski, Katherine Lee, Florian Tramer, and Chiyuan Zhang.
\newblock Quantifying memorization across neural language models.
\newblock In \emph{International Conference on Learning Representations}, 2022.

\bibitem[Chen et~al.(2024)Chen, Lou, Huang, and Zhang]{chen2024defining}
Lu~Chen, Siyu Lou, Benhao Huang, and Quanshi Zhang.
\newblock Defining and extracting generalizable interaction primitives from dnns.
\newblock In \emph{International Conference on Learning Representations}, 2024.

\bibitem[Cobbe et~al.(2021)Cobbe, Kosaraju, Bavarian, Chen, Jun, Kaiser, Plappert, Tworek, Hilton, Nakano, et~al.]{cobbe2021training}
Karl Cobbe, Vineet Kosaraju, Mohammad Bavarian, Mark Chen, Heewoo Jun, Lukasz Kaiser, Matthias Plappert, Jerry Tworek, Jacob Hilton, Reiichiro Nakano, et~al.
\newblock Training verifiers to solve math word problems.
\newblock \emph{arXiv preprint arXiv:2110.14168}, 2021.

\bibitem[Feldman and Zhang(2020)]{feldman2020neural}
Vitaly Feldman and Chiyuan Zhang.
\newblock What neural networks memorize and why: Discovering the long tail via influence estimation.
\newblock In \emph{Advances in Neural Information Processing Systems}, volume~33, pages 2881--2891, 2020.

\bibitem[Geva et~al.(2021)Geva, Khashabi, Segal, Khot, Roth, and Berant]{geva2021did}
Mor Geva, Daniel Khashabi, Elad Segal, Tushar Khot, Dan Roth, and Jonathan Berant.
\newblock Did aristotle use a laptop? a question answering benchmark with implicit reasoning strategies.
\newblock \emph{Transactions of the Association for Computational Linguistics}, 9:\penalty0 346--361, 2021.

\bibitem[Li and Zhang(2023{\natexlab{a}})]{li2023defining}
Mingjie Li and Quanshi Zhang.
\newblock Defining and quantifying and-or interactions for faithful and concise explanation of dnns.
\newblock \emph{arXiv preprint arXiv:2304.13312}, 2023{\natexlab{a}}.

\bibitem[Li and Zhang(2023{\natexlab{b}})]{li2023does}
Mingjie Li and Quanshi Zhang.
\newblock Does a neural network really encode symbolic concepts?
\newblock In \emph{International Conference on Machine Learning}, pages 20452--20469. PMLR, 2023{\natexlab{b}}.

\bibitem[Liu et~al.(2021)Liu, Cui, Liu, Huang, Wang, and Zhang]{liu2021logiqa}
Jian Liu, Leyang Cui, Hanmeng Liu, Dandan Huang, Yile Wang, and Yue Zhang.
\newblock Logiqa: a challenge dataset for machine reading comprehension with logical reasoning.
\newblock In \emph{Proceedings of the Twenty-Ninth International Conference on International Joint Conferences on Artificial Intelligence}, pages 3622--3628, 2021.

\bibitem[McCoy et~al.(2023)McCoy, Smolensky, Linzen, Gao, and Celikyilmaz]{mccoy2023much}
R~Thomas McCoy, Paul Smolensky, Tal Linzen, Jianfeng Gao, and Asli Celikyilmaz.
\newblock How much do language models copy from their training data? evaluating linguistic novelty in text generation using raven.
\newblock \emph{Transactions of the Association for Computational Linguistics}, 11:\penalty0 652--670, 2023.

\bibitem[Nie et~al.(2020)Nie, Williams, Dinan, Bansal, Weston, and Kiela]{nie2020adversarial}
Yixin Nie, Adina Williams, Emily Dinan, Mohit Bansal, Jason Weston, and Douwe Kiela.
\newblock Adversarial nli: A new benchmark for natural language understanding.
\newblock In \emph{Proceedings of the 58th Annual Meeting of the Association for Computational Linguistics}, pages 4885--4901, 2020.

\bibitem[Ren et~al.(2023{\natexlab{a}})Ren, Li, Chen, Deng, and Zhang]{ren2023defining}
Jie Ren, Mingjie Li, Qirui Chen, Huiqi Deng, and Quanshi Zhang.
\newblock Defining and quantifying the emergence of sparse concepts in dnns.
\newblock In \emph{Proceedings of the IEEE/CVF Conference on Computer Vision and Pattern Recognition}, pages 20280--20289, 2023{\natexlab{a}}.

\bibitem[Ren et~al.(2023{\natexlab{b}})Ren, Zhou, Chen, and Zhang]{ren2023can}
Jie Ren, Zhanpeng Zhou, Qirui Chen, and Quanshi Zhang.
\newblock Can we faithfully represent absence states to compute shapley values on a dnn?
\newblock In \emph{International Conference on Learning Representations}, 2023{\natexlab{b}}.

\bibitem[Ren et~al.(2024)Ren, Gao, Shen, and Zhang]{ren2024where}
Qihan Ren, Jiayang Gao, Wen Shen, and Quanshi Zhang.
\newblock Where we have arrived in proving the emergence of sparse symbolic concepts in ai models.
\newblock In \emph{International Conference on Learning Representations}, 2024.

\bibitem[Shen et~al.(2023)Shen, Cheng, Yang, Li, and Zhang]{shen2023can}
Wen Shen, Lei Cheng, Yuxiao Yang, Mingjie Li, and Quanshi Zhang.
\newblock Can the inference logic of large language models be disentangled into symbolic concepts?
\newblock \emph{arXiv preprint arXiv:2304.01083}, 2023.

\bibitem[Talmor et~al.(2019)Talmor, Herzig, Lourie, and Berant]{talmor2019commonsenseqa}
Alon Talmor, Jonathan Herzig, Nicholas Lourie, and Jonathan Berant.
\newblock Commonsenseqa: A question answering challenge targeting commonsense knowledge.
\newblock In \emph{Proceedings of NAACL-HLT}, pages 4149--4158, 2019.

\bibitem[Tenney et~al.(2020)Tenney, Wexler, Bastings, Bolukbasi, Coenen, Gehrmann, Jiang, Pushkarna, Radebaugh, Reif, et~al.]{tenney2020language}
Ian Tenney, James Wexler, Jasmijn Bastings, Tolga Bolukbasi, Andy Coenen, Sebastian Gehrmann, Ellen Jiang, Mahima Pushkarna, Carey Radebaugh, Emily Reif, et~al.
\newblock The language interpretability tool: Extensible, interactive visualizations and analysis for nlp models.
\newblock In \emph{Proceedings of the 2020 Conference on Empirical Methods in Natural Language Processing: System Demonstrations}, pages 107--118, 2020.

\bibitem[Tian et~al.(2021)Tian, Li, Chen, Xiao, He, and Jin]{tian2021diagnosing}
Jidong Tian, Yitian Li, Wenqing Chen, Liqiang Xiao, Hao He, and Yaohui Jin.
\newblock Diagnosing the first-order logical reasoning ability through logicnli.
\newblock In \emph{Proceedings of the 2021 Conference on Empirical Methods in Natural Language Processing}, pages 3738--3747, 2021.

\bibitem[Tirumala et~al.(2022)Tirumala, Markosyan, Zettlemoyer, and Aghajanyan]{tirumala2022memorization}
Kushal Tirumala, Aram Markosyan, Luke Zettlemoyer, and Armen Aghajanyan.
\newblock Memorization without overfitting: Analyzing the training dynamics of large language models.
\newblock \emph{Advances in Neural Information Processing Systems}, 35:\penalty0 38274--38290, 2022.

\bibitem[Touvron et~al.(2023)Touvron, Lavril, Izacard, Martinet, Lachaux, Lacroix, Rozi{\`e}re, Goyal, Hambro, Azhar, et~al.]{touvron2023llama}
Hugo Touvron, Thibaut Lavril, Gautier Izacard, Xavier Martinet, Marie-Anne Lachaux, Timoth{\'e}e Lacroix, Baptiste Rozi{\`e}re, Naman Goyal, Eric Hambro, Faisal Azhar, et~al.
\newblock Llama: Open and efficient foundation language models.
\newblock \emph{arXiv preprint arXiv:2302.13971}, 2023.

\bibitem[Wei et~al.(2022{\natexlab{a}})Wei, Tay, Bommasani, Raffel, Zoph, Borgeaud, Yogatama, Bosma, Zhou, Metzler, et~al.]{wei2022emergent}
Jason Wei, Yi~Tay, Rishi Bommasani, Colin Raffel, Barret Zoph, Sebastian Borgeaud, Dani Yogatama, Maarten Bosma, Denny Zhou, Donald Metzler, et~al.
\newblock Emergent abilities of large language models.
\newblock \emph{Transactions on Machine Learning Research}, 2022{\natexlab{a}}.

\bibitem[Wei et~al.(2022{\natexlab{b}})Wei, Wang, Schuurmans, Bosma, Xia, Chi, Le, Zhou, et~al.]{wei2022chain}
Jason Wei, Xuezhi Wang, Dale Schuurmans, Maarten Bosma, Fei Xia, Ed~Chi, Quoc~V Le, Denny Zhou, et~al.
\newblock Chain-of-thought prompting elicits reasoning in large language models.
\newblock \emph{Advances in neural information processing systems}, 35:\penalty0 24824--24837, 2022{\natexlab{b}}.

\bibitem[Yu et~al.(2019)Yu, Jiang, Dong, and Feng]{yu2019reclor}
Weihao Yu, Zihang Jiang, Yanfei Dong, and Jiashi Feng.
\newblock Reclor: A reading comprehension dataset requiring logical reasoning.
\newblock In \emph{International Conference on Learning Representations}, 2019.

\bibitem[Zhang et~al.(2024)Zhang, Ippolito, Lee, Jagielski, Tram{\`e}r, and Carlini]{zhang2024counterfactual}
Chiyuan Zhang, Daphne Ippolito, Katherine Lee, Matthew Jagielski, Florian Tram{\`e}r, and Nicholas Carlini.
\newblock Counterfactual memorization in neural language models.
\newblock In \emph{Advances in Neural Information Processing Systems}, volume~36, 2024.

\bibitem[Zhang et~al.(2022)Zhang, Roller, Goyal, Artetxe, Chen, Chen, Dewan, Diab, Li, Lin, et~al.]{zhang2022opt}
Susan Zhang, Stephen Roller, Naman Goyal, Mikel Artetxe, Moya Chen, Shuohui Chen, Christopher Dewan, Mona Diab, Xian Li, Xi~Victoria Lin, et~al.
\newblock Opt: Open pre-trained transformer language models.
\newblock \emph{arXiv preprint arXiv:2205.01068}, 2022.

\bibitem[Zhou et~al.(2023)Zhou, Tang, Li, Zhang, Liu, and Zhang]{zhou2023explaining}
Huilin Zhou, Huijie Tang, Mingjie Li, Hao Zhang, Zhenyu Liu, and Quanshi Zhang.
\newblock Explaining how a neural network play the go game and let people learn.
\newblock \emph{arXiv preprint arXiv:2310.09838}, 2023.

\bibitem[Zhou et~al.(2024)Zhou, Zhang, Deng, Liu, Shen, Chan, and Zhang]{zhou2024explaining}
Huilin Zhou, Hao Zhang, Huiqi Deng, Dongrui Liu, Wen Shen, Shih-Han Chan, and Quanshi Zhang.
\newblock Explaining generalization power of a dnn using interactive concepts.
\newblock In \emph{Proceedings of the AAAI Conference on Artificial Intelligence}, volume~38, pages 17105--17113, 2024.

\end{thebibliography}
\appendix

\newpage

\section{Appendix}
\subsection{Related work}
Previous research has extensively explored the qualitative aspects of memorization and reasoning effects in the inference of LLMs. For example, previous studies~\cite{carlini2019secret,carlini2021extracting, mccoy2023much, tirumala2022memorization} have demonstrated that it was possible to let LLM output memorized data, such as phone numbers and usernames, when appropriate prompts are provided.  Recent research has suggested that reasoning ability may emerge in LLMs~\cite{wei2022emergent, wei2022chain}. Evaluating the reasoning capabilities of LLMs often involves reporting their performance, \emph{e.g.}, accuracy, on specific tasks that require reasoning. Common benchmarks for assessing LLM's reasoning capabilities include tasks related to arithmetic reasoning~\cite{cobbe2021training, amini2019mathqa} and commonsense reasoning~\cite{talmor2019commonsenseqa, geva2021did}. While LLMs have demonstrated remarkable performance on various memorization and reasoning tasks, the degree to which their predictions rely on reasoning or memorization remains unclear. Existing studies typically focus on end-task accuracy rather than directly analyzing and quantifying the exact memorization and reasoning effects within the inference process. In this work, we aim to disentangle and quantify the inference logic into memorization effects and reasoning effects\footnote{There are no widely accepted boundaries between memorization and reasoning. Therefore, to facilitate our analysis, we focus on context-agnostic memorization and in-context reasoning, as detailed in Section~\ref{sec:mvsr}.}.

However, before disentangling the inference logic into specific memorization effects and reasoning effects. We need to define the explicit logic used by the LLM. To this end, \citet{ren2023defining} have proposed to use interactions between input variables encoded by the deep neural network (DNN) to explain the detailed logic in the network. Subsequently,~\cite{ren2024where} have mathematically proven that DNNs often only encode a small number of interactions, and the output score of the DNN on any random masked input sample can always be well approximated by numerical effects of these interactions. Moreover, a set of desirable properties of interactions have been further discovered, such as transferability~\cite{li2023does}, discriminative power~\cite{li2023does} and generalization power~\cite{zhou2023explaining}.

\subsection{Extraction of AND-OR interactions}\label{append:extract}
The inference logic of neural networks is often very complex, therefore it is difficult to solely rely on one type of interaction (whether it's AND interaction or OR interaction) to concisely and faithfully explain the true inference primitives encoded by the DNN. To address this issue,~\citet{zhou2023explaining} proposed to decompose the output $v(\mathbf{x}_T)$ of the neural network into two parts: one part modeled by AND interactions, denoted as $v_{\text{and}}(x_T) = 0.5v(\mathbf{x}_T) + \theta_T$, and the other part modeled by OR interactions, denoted as $v_{\text{or}}(x_T) = 0.5v(\mathbf{x}_T) - \theta_T$. Here, $\{\theta_T\}$ was a set of learnable parameters. Therefore, finding the most suitable decomposition between $v_{\text{and}}(\mathbf{x}_T)$ and $v_{\text{or}}(\mathbf{x}_T)$ was equivalent to determining the optimal $\{\theta_T\}$, where $\theta_T \in \mathbb{R}$. \citet{zhou2023explaining} proposed to learn the optimal parameters $\{\theta_T\}$ by minimizing the $L_1$ norm of AND and OR interactions, aiming to generate the sparsest AND-OR interactions.
\begin{equation}
    \text{minimize}_{\{\theta_T\}} \|\mathbf{I}_{\text{and}}\|_1 + \|\mathbf{I}_{\text{or}}\|_1,
\end{equation}
where $\mathbf{I}_{\text{and}} = [I_{\text{and}}(T_1|\mathbf{x}), \dots, I_{\text{and}}(T_{2^n}|\mathbf{x})]^\top$ and $\mathbf{I}_{\text{or}} = [I_{\text{or}}(T_1|\mathbf{x}),  \dots, I_{\text{or}}(T_{2^n}|\mathbf{x})]^\top$, $T_k \subseteq N$.

\textbf{Modeling noises.} \citet{chen2024defining} found that in practical application scenarios, there inevitably existed some noises in the network's outputs, which might not be well explained using AND-OR interactions. Assuming that there were small noises in the network output $v(\mathbf{x}_T)$. Such noises were represented by adding Gaussian noise $ \epsilon_T \sim \mathcal{N}(0, \sigma^2)$ with small variance to the network output, \emph{i.e.}, $v'_{\text{and}}(\mathbf{x}_T) = v_{\text{and}}(\mathbf{x}_T) + \epsilon_T$ and $v'_{\text{or}}(\mathbf{x}_T) = v_{\text{or}}(\mathbf{x}_T) + \epsilon_T$.  \citet{chen2024defining} further proposed to directly learn the error term $\epsilon_T$ to remove the noisy signals, \emph{i.e.}, setting $v(\mathbf{x}_T) = v_{\text{and}}(\mathbf{x}_T) + v_{\text{or}}(x_T) + \epsilon_T$.

\subsection{Proofs for $v_{\text{and}}(\mathbf{x}_T)=\sum_{S\subseteq T}I_{\text{and}}(S|\mathbf{x})$ and $v_{\text{or}}(\mathbf{x}_T)=\sum_{S\cap T\neq\emptyset}I_{\text{or}}(S|\mathbf{x})$}\label{appendix:poofforvcomponent}

In this subsection, we will first prove that the component $v_{\text{and}}(\mathbf{x}_T)$ of the neural network can be solely modeled by AND interactions, \emph{i.e.}, $\forall T\subseteq N$, $v_{\text{and}}(\mathbf{x}_T)=\sum_{S\subseteq T}I_{\text{and}}(S|\mathbf{x})$.

\begin{proof}
    According to Equation~(\ref{eq:andor}), the AND interaction is defined as $I_{\text{and}}(S|\mathbf{x}) = \sum_{L\subseteq S}(-1)^{|S|-|L|}v_{\text{and}}(\mathbf{x}_L)$, hence, 
    \begin{equation}
        \begin{aligned}
            \sum\nolimits_{S\subseteq T}I_{\text{and}}(S|\mathbf{x}) &= \sum\nolimits_{S\subseteq T}\sum\nolimits_{L\subseteq S}(-1)^{|S|-|L|}v_{\text{and}}(\mathbf{x}_L) \\
            & = \sum\nolimits_{L\subseteq T}\sum\nolimits_{S:L\subseteq S\subseteq T}(-1)^{|S|-|L|}v_{\text{and}}(\mathbf{x}_L)\\
            & = v_{\text{and}}(\mathbf{x}_T) +\sum\nolimits_{L\subset T}v_{\text{and}}(\mathbf{x}_L)\cdot \underbrace{\sum\nolimits_{m=0}^{|T|-|L|}\left(\begin{array}{cc}
                 |T|-|L|  \\
                 m 
            \end{array} \right)(-1)^m}_{=0} \\
            & = v_{\text{and}}(\mathbf{x}_T). 
        \end{aligned}   
    \end{equation}
\end{proof}

Then we will prove that the network output component $v_{\text{or}}(\mathbf{x}_T)$ can be solely explained by OR interactions, \emph{i.e.},  $\forall T\subseteq N$, $v_{\text{or}}(\mathbf{x}_T)=\sum_{S\cap T\neq \emptyset}I_{\text{or}}(S|\mathbf{x})$.

\begin{proof}
    According to Equation~(\ref{eq:andor}), the OR interaction is defined as $I_{\text{or}}(S|\mathbf{x}) = -\sum\nolimits_{L\subseteq S, S\neq \emptyset}(-1)^{|S|-|L|}v_{\text{\rm or}}(\mathbf{x}_{N\setminus L})$ and $I_{\text{or}}(\emptyset|\mathbf{x})=v_{\text{or}}(\mathbf{x}_\emptyset)$, therefore, 
    \begin{equation}
        \begin{aligned}
          \sum_{S:S\cap T\neq \emptyset}I_{\text{or}}(S|\mathbf{x})  =& \sum\nolimits_{S:S\cap T\neq \emptyset}\left[-\sum\nolimits_{L\subseteq S, S\neq \emptyset}(-1)^{|S|-|L|}v_{\text{\rm or}}(\mathbf{x}_{N\setminus L})\right] +v_{\text{or}}(\mathbf{x}_\emptyset)\\
        =& -\sum\nolimits_{L\subseteq N}\sum\nolimits_{S:S\cap T \neq \emptyset, S\supseteq L}(-1)^{|S|-|L|}v_{\text{or}}(\mathbf{x}_{N\setminus L}) +v_{\text{or}}(\mathbf{x}_\emptyset)\\
          =& - \underbrace{v_{\text{or}}(\mathbf{x}_\emptyset)}_{L=N} - \underbrace{v_{\text{or}}(\mathbf{x}_T)}_{L=N\setminus T}\cdot \underbrace{\sum\nolimits_{|S'|=1}^{|T|} \left(\begin{array}{cc}
               |T|\\
               |S'|
          \end{array}\right)(-1)^{|S'|}}_{=-1} \\
        & -\!\!\!\!\!\!\!\! \sum_{L\cap T\neq\emptyset, L\neq N}\sum_{S'\subseteq N\setminus T \setminus L}\underbrace{\sum_{|S''|=0}^{|T|-|T\cap L|}\left(\begin{array}{cc}
             |T|-|T\cap L| \\
             |S''|
        \end{array}\right)(-1)^{|S'|+|S''|}}_{=0}\cdot v_{\text{or}}(\mathbf{x}_{N\setminus L}) \\
        & - \!\!\!\!\!\!\!\!\sum_{L\cap T=\emptyset, L\neq N\setminus T}\sum_{S'\subseteq N\setminus T\setminus L}\underbrace{\sum_{|S''|=0}^{|T|}\left(\begin{array}{cc}
             |T|  \\
             |S''| 
        \end{array}\right)(-1)^{|S'|+|S''|}}_{=0} \cdot v_{\text{or}}(\mathbf{x}_{N\setminus L})  + v_{\text{or}}(\mathbf{x}_\emptyset) \\
        = & v_{\text{or}}(\mathbf{x}_T).
        \end{aligned}
    \end{equation}
\end{proof}

\subsection{Sparsity property of interactions} \label{appendix:sparsity}
Given a sample with $n$ input variables (tokens), among all $2^n$ possible interactions, only approximately $\mathcal{O}(n^\kappa / \tau)$ interactions have salient interaction effects~\cite{ren2024where}. Empirically, $\kappa\in[0.9,1.2]$, so that the number of salient interactions $\Gamma$ is much less than $2^n$, \emph{i.e.}, the interactions are sparse.

We conducted experiments on OPT-1.3B, LLaMA-7B and GPT-3.5-Turbo. We used the prompts in our question-answer dataset (see Appendix~\ref{appendix:prompts}) and let the LLM generate the next word. We used the method described in Appendix~\ref{append:extract} to extract AND-OR interactions. Figure~\ref{fig:fig6} shows all interaction effects encoded by the LLM on different input prompts. It shows that there are only a few salient interactions extracted from an input prompt, and all other interactions have negligible effects on the LLM's output. 

\textbf{Sparsity property of decomposed effects.} We also conducted experiments to verify the sparsity of the decomposed effects. Each interaction effect $I_{\text{and}}(S|\mathbf{x})$ can be decomposed into three effects, foundational memorization effect $\mathcal{J}^{\text{f}}_{\text{and}}(S|\mathbf{x})$, chaotic memorization effect $\mathcal{J}^{\text{c}}_{\text{and}}(S|\mathbf{x})$ and in-context reasoning effect $\mathcal{K}_{\text{and}}(S|\mathbf{x})$. And we can rewrite the LLM's confidence score into the sum of the above three effects (see Figure~\ref{eq:eq7}). We sorted all effects in descending order of their strength, denoted as $\mathcal{E}(S|\mathbf{x})$, and Figure~\ref{fig:fig7} shows there are only a few effects that are salient, and all other effects have negligible effects on the LLM's output.

\begin{figure}[t]
    \centering
    \includegraphics[width=0.9\textwidth]{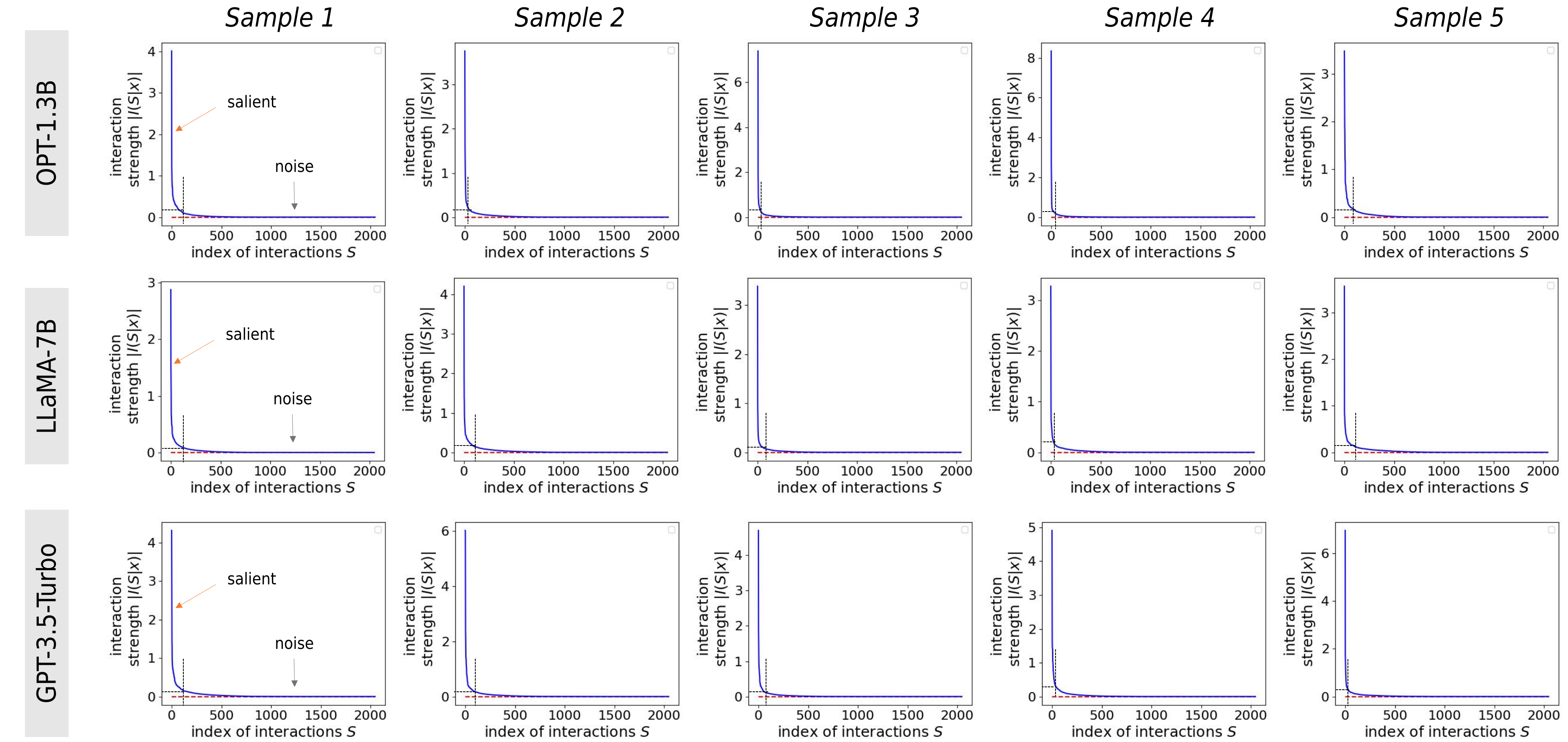}
    \caption{Interactions in descending order of interaction strength $\vert I(S|\mathbf{x})\vert$. }
    \label{fig:fig6}
\end{figure}

\begin{figure}[t]
    \centering
    \includegraphics[width=0.9\textwidth]{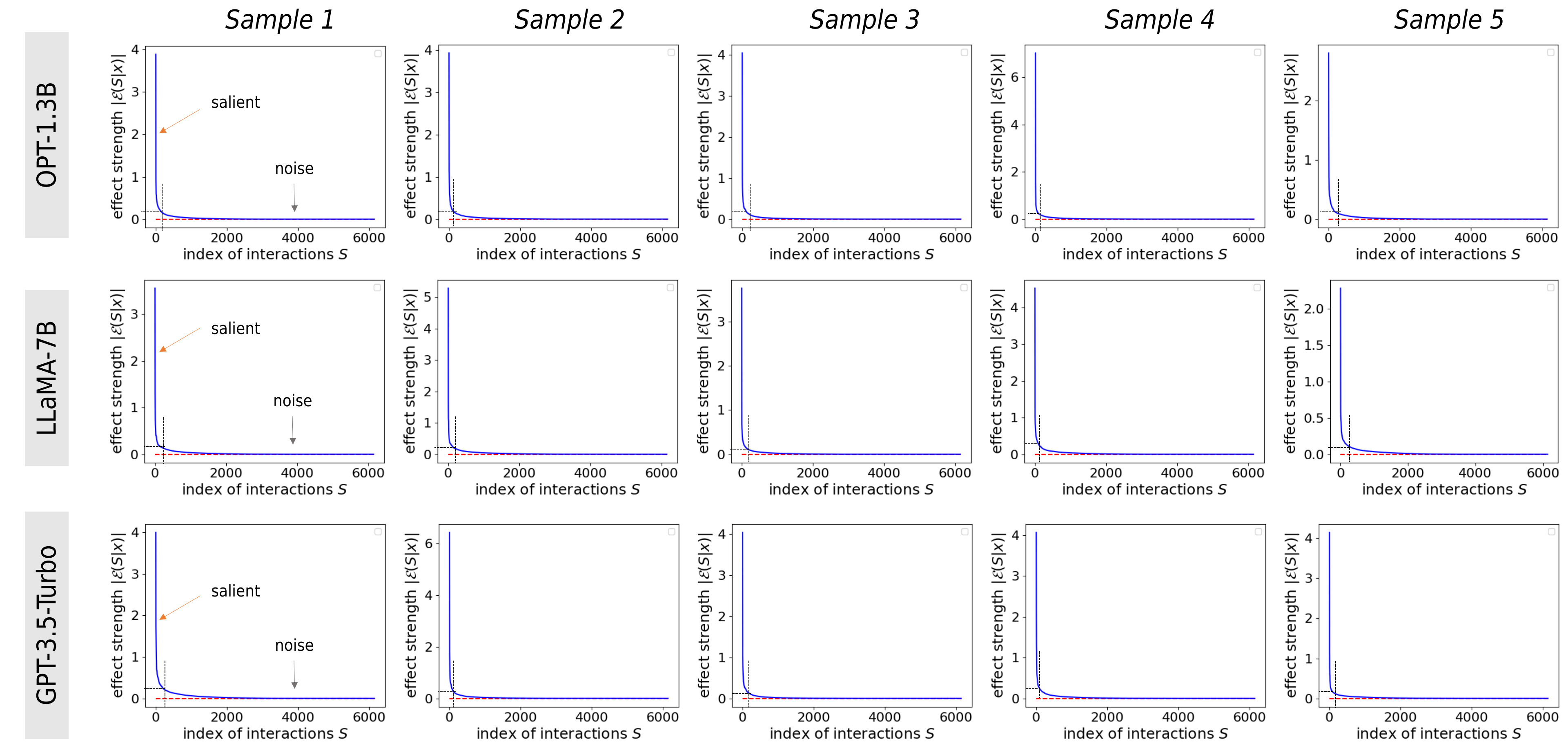}
    \caption{Decomposed effects in descending order of their strength $\vert \mathcal{E}(S|\mathbf{x})\vert$. }
    \label{fig:fig7}
\end{figure}

\subsection{Explanations of how the OR interactions can be considered as a specific AND interaction}\label{appendix:or-interaction}
The OR interaction can be computed as follows,
\begin{equation}
    I_{\text{or}}(S|\mathbf{x}) = -\sum\nolimits_{T\subseteq S, S\neq \emptyset}(-1)^{|S|-|T|}v(\mathbf{x}_{N\setminus T}),
\end{equation}
where $\mathbf{x}_{N\setminus T}$ denotes a masked sample where the variables in $N\setminus T$ are not masked and the variables in $T$ are masked using the corresponding baseline values  $\mathbf{b}$. 
Now, we consider $b_i$ as the presence of the $i$-th variable and $x_i$ as the masked state of the $i$-th variable, and denoted such masked sample as $\mathbf{x}'_{T}$. In this way,  $\mathbf{x}_{N\setminus T}$ represents the same masked state as $\mathbf{x}'_{T}$, then the OR interaction can be formulated as a specific AND interaction in Equation~(\ref{eq:andor}) as follows,
\begin{align*}
    I_{\text{or}}(S|\mathbf{x}) &= -\sum\nolimits_{T\subseteq S, S\neq \emptyset}(-1)^{|S|-|T|}v(\mathbf{x}_{N\setminus T}) \\
    & = -\sum\nolimits_{T\subseteq S, S\neq \emptyset}(-1)^{|S|-|T|}v(\mathbf{x}'_{T}) \\
    & = -I'_{\text{and}}(S|\mathbf{x}').
\end{align*}
Hence, we can consider an OR interaction as a specific AND interaction.

\subsection{Common conditions for the sparsity of interactions encoded by a DNN}\label{appendix:condition}
\citet{ren2024where} have proved that if an DNN satisfies the following three conditions, it usually encodes a small number of interactions. Furthermore, these interactions can very well approximate the DNN's output across $2^n$ arbitrary masked samples $\{\mathbf{x}_S|S\subseteq N\}$.
\begin{enumerate}
    \item The DNN does not encoded interactions with very high order (extremely complex interactions). In other words, interactions with extremely many input variables have zero numerical contributions to the DNN's output. 
    \item When we mask different random sets of $m$ input variables and obtain masked samples $\{\mathbf{x}_S\vert |S|=n-m\}$, then the average network output $\bar{u}^{(m)}=\mathbb{E}_{|S|=m}[v(\mathbf{x}_S)-v(\mathbf{x}_\emptyset)]$ monotonically decreases with $m$. 
    \item The decreasing speed of the average network output $\bar{u}^{(m)}$ is polynomial, \emph{i.e.}, $\forall m'<m, \bar{u}^{(m')}\geq (\frac{m'}{m})^p\bar{u}^{(m)}$, where $p>0$ is a postive constant.
\end{enumerate}

\subsection{Question-answer dataset and their logically equivalent prompts} \label{appendix:prompts}
We constructed 10 prompts for the question-answer (QA) task to facilitate the analysis and visualization based on the proposed axiomatic system, as shown in Table~\ref{tab:logical_equivalent_full}. Each sample, denoted as the original prompt in Table~\ref{tab:logical_equivalent_full}, comprises a premise $\mathbf{x}_p$ and a question $\mathbf{x}_q$, where the question does not contain any reasoning clues. For every sample, we constructed nine logically equivalent prompts across three categories: removing and adding irrelevant background, semantic paraphrasing and renaming. 
{\footnotesize
\LTcapwidth=\textwidth
\begin{longtable}{l|p{0.7\linewidth}}
    \toprule
        Type &  Logically equivalent prompt \\
    \midrule
    \midrule
        \textbf{Original prompt} & \textcolor{orange}{Mary is the wife of John.} \textcolor{teal}{\underline{James} \underline{is} \underline{the} \underline{son} \underline{of} \underline{John}. \underline{Mary} is \underline{the} \underline{mother} \underline{of}}\\
        \midrule
         \multirow{2}{*}{\makecell{Removing and adding \\irrelevant background}} & a. \textcolor{orange}{Mary is the wife of John and Tina is the sister of Mary.} \textcolor{teal}{James is the son of John. Mary is the mother of}\\
         & b. \textcolor{orange}{Mary is the wife of John and she likes cooking.}. \textcolor{teal}{James is the son of John. Mary is the mother of} \\
         \midrule
         \multirow{2}{*}{Semantic paraphrasing} & a. \textcolor{orange}{John and Mary are married.} \textcolor{teal}{James is the son of John. Mary is the mother of} \\
         &b. \textcolor{orange}{John is the husband of Mary.} \textcolor{teal}{James is the son of John. Mary is the mother of} \\
         \midrule
         \multirow{5}{*}{Renaming} &a. \textcolor{orange}{Mary is the wife of Tim.} \textcolor{teal}{James is the son of Tim. Mary is the mother of} \\
         &b. \textcolor{orange}{Mary is the wife of William.} \textcolor{teal}{James is the son of William. Mary is the mother of} \\ 
         &c. \textcolor{orange}{Mary is the wife of Jack.} \textcolor{teal}{James is the son of Jack. Mary is the mother of} \\ 
         &d. \textcolor{orange}{Mary is the wife of Peter.} \textcolor{teal}{James is the son of Peter. Mary is the mother of} \\ 
         &e. \textcolor{orange}{Mary is the wife of Michael.} \textcolor{teal}{James is the son of Michael. Mary is the mother of} \\ 
    \midrule
    \midrule
    \textbf{Original prompt}  & \textcolor{orange}{David has 1 month of vocations.} \textcolor{teal}{\underline{David} \underline{started} the \underline{vocation} in the \underline{beginning} of \underline{June}. \underline{He} \underline{will} \underline{finish} his \underline{vocation} in the \underline{beginning} of}\\
        \midrule
         \multirow{2}{*}{\makecell{Removing and adding \\irrelevant background}} & a. \textcolor{orange}{David has 1 month of vocations and Sarah only has one day of vocation.} \textcolor{teal}{David started the vocation in the beginning of June. He will finish his vocation in the beginning of}\\
         & b. \textcolor{orange}{David has 1 month of vocations and he plans to travel.}. \textcolor{teal}{David started the vocation in the beginning of June. He will finish his vocation in the beginning of} \\
         \midrule
         \multirow{2}{*}{Semantic paraphrasing} & a. \textcolor{orange}{David has 30 days of vocations.} \textcolor{teal}{David started the vocation in the beginning of June. He will finish his vocation in the beginning of} \\
         &b. \textcolor{orange}{David has 1 month of holidays.} \textcolor{teal}{David started the vocation in the beginning of June. He will finish his vocation in the beginning of} \\
         \midrule
         \multirow{5}{*}{Renaming} &a. \textcolor{orange}{Tim has 1 month of vocations.} \textcolor{teal}{Tim started the vocation in the beginning of June. He will finish his vocation in the beginning of} \\
         &b. \textcolor{orange}{William has 1 month of vocations.} \textcolor{teal}{William started the vocation in the beginning of June. He will finish his vocation in the beginning of} \\ 
         &c. \textcolor{orange}{Peter has 1 month of vocations.} \textcolor{teal}{Peter started the vocation in the beginning of June. He will finish his vocation in the beginning of} \\ 
         &d. \textcolor{orange}{John has 1 month of vocations.} \textcolor{teal}{John started the vocation in the beginning of June. He will finish his vocation in the beginning of} \\ 
         &e. \textcolor{orange}{Jack has 1 month of vocations.} \textcolor{teal}{Jack started the vocation in the beginning of June. He will finish his vocation in the beginning of} \\
    \midrule
    \midrule
    \textbf{Original prompt}  & \textcolor{orange}{John likes cakes.} \textcolor{teal}{\underline{There} \underline{is} a \underline{cake} \underline{on} the \underline{table}. \underline{It} \underline{will} \underline{be} \underline{eaten} \underline{by}}\\
        \midrule
         \multirow{2}{*}{\makecell{Removing and adding \\irrelevant background}} & a. \textcolor{orange}{John likes cakes and Mary does not like cakes.} \textcolor{teal}{There is a cake on the table. It will be eaten by}\\
         & b. \textcolor{orange}{John likes cakes and noodles.}. \textcolor{teal}{There is a cake on the table. It will be eaten by} \\
         \midrule
         \multirow{2}{*}{Semantic paraphrasing} & a. \textcolor{orange}{John takes pleasure in eating cakes.} \textcolor{teal}{There is a cake on the table. It will be eaten by} \\
         &b. \textcolor{orange}{John loves cakes.} \textcolor{teal}{There is a cake on the table. It will be eaten by} \\
         \midrule
         \multirow{5}{*}{Renaming} &a. \textcolor{orange}{John likes candies.} \textcolor{teal}{There is a candy on the table. It will be eaten by} \\
         &b. \textcolor{orange}{John likes peaches.} \textcolor{teal}{There is a peach on the table. It will be eaten by} \\ 
         &c. \textcolor{orange}{John likes pineapple.} \textcolor{teal}{There is a pineapple on the table. It will be eaten by} \\ 
         &d. \textcolor{orange}{John likes cookies.} \textcolor{teal}{There is a cookie on the table. It will be eaten by} \\ 
         &e. \textcolor{orange}{John likes bananas.} \textcolor{teal}{There is a banana on the table. It will be eaten by} \\ 
    \midrule
    \midrule
    \textbf{Original prompt}  & \textcolor{orange}{Peter is a pianist.} \textcolor{teal}{\underline{Emily} is the \underline{apprentice} of \underline{Peter} and \underline{practices} \underline{with} \underline{him} \underline{everyday}. \underline{Emily} \underline{plays} \underline{the}}\\
        \midrule
         \multirow{2}{*}{\makecell{Removing and adding \\irrelevant background}} & a. \textcolor{orange}{Peter is a pianist and Anna is a violinist.} \textcolor{teal}{Emily is the apprentice of Peter and practices with him everyday. Emily plays the}\\
         & b. \textcolor{orange}{Peter is a pianist and he goes often to swim.}. \textcolor{teal}{Emily is the apprentice of Peter and practices with him everyday. Emily plays the} \\
         \midrule
         \multirow{2}{*}{Semantic paraphrasing} & a. \textcolor{orange}{Peter teaches the piano.} \textcolor{teal}{Emily is the apprentice of Peter and practices with him everyday. Emily plays the} \\
         &b. \textcolor{orange}{Peter is skilled at playing the piano.} \textcolor{teal}{Emily is the apprentice of Peter and practices with him everyday. Emily plays the} \\
         \midrule
         \multirow{5}{*}{Renaming} &a. \textcolor{orange}{Adam is skilled at playing the piano.} \textcolor{teal}{Emily is the apprentice of Adam and practices with him everyday. Emily plays the} \\
         &b. \textcolor{orange}{Billy is skilled at playing the piano.} \textcolor{teal}{Emily is the apprentice of Billy and practices with him everyday. Emily plays the} \\ 
         &c. \textcolor{orange}{Ben is skilled at playing the piano.} \textcolor{teal}{Emily is the apprentice of Ben and practices with him everyday. Emily plays the} \\ 
         &d. \textcolor{orange}{Tobias is skilled at playing the piano.} \textcolor{teal}{Emily is the apprentice of Tobias and practices with him everyday. Emily plays the} \\ 
         &e. \textcolor{orange}{Tim is skilled at playing the piano.} \textcolor{teal}{Emily is the apprentice of Tim and practices with him everyday. Emily plays the} \\ 
    \midrule
    \midrule
   \textbf{Original prompt}  & \textcolor{orange}{Julia is having a birthday party today.} \textcolor{teal}{\underline{Anna} \underline{is} \underline{leaving} \underline{work} \underline{early} \underline{today}. \underline{She} \underline{plans} \underline{to} \underline{visit}}\\
        \midrule
         \multirow{2}{*}{\makecell{Removing and adding \\irrelevant background}} & a. \textcolor{orange}{Today is Julia's birthday and she is a friend of Anna.} \textcolor{teal}{Anna is leaving work early today. She plans to visit}\\
         & b. \textcolor{orange}{Julia invited Anna to her birthday party today and they are very good friend.}. \textcolor{teal}{Anna is leaving work early today. She plans to visit} \\
         \midrule
         \multirow{2}{*}{Semantic paraphrasing} & a. \textcolor{orange}{Anna is invited by Julia. } \textcolor{teal}{Anna is leaving work early today. She plans to visit} \\
         &b. \textcolor{orange}{Anna and Julia have an appointment today.} \textcolor{teal}{Anna is leaving work early today. She plans to visit} \\
         \midrule
         \multirow{5}{*}{Renaming} &a. \textcolor{orange}{Julia is having a birthday party today.} \textcolor{teal}{Tina is leaving work early today. She plans to visit} \\
         &b. \textcolor{orange}{Julia is having a birthday party today.} \textcolor{teal}{Emma is leaving work early today. She plans to visit} \\ 
         &c. \textcolor{orange}{Julia is having a birthday party today.} \textcolor{teal}{Cynthia is leaving work early today. She plans to visit} \\ 
         &d. \textcolor{orange}{Julia is having a birthday party today.} \textcolor{teal}{Emily is leaving work early today. She plans to visit} \\ 
         &e. \textcolor{orange}{Julia is having a birthday party today.} \textcolor{teal}{Kate is leaving work early today. She plans to visit} \\ 
    \midrule
    \midrule
    \textbf{Original prompt}  & \textcolor{orange}{Caren works as a teacher.} \textcolor{teal}{\underline{Emily} \underline{is} \underline{the} \underline{colleague} \underline{of} \underline{Caren}. \underline{Emily} \underline{works} \underline{as} \underline{a}}\\
        \midrule
         \multirow{2}{*}{\makecell{Removing and adding \\irrelevant background}} & a. \textcolor{orange}{Caren works as a teacher and Tom works as a doctor.} \textcolor{teal}{Emily is the colleague of Caren. Emily works as a}\\
         & b. \textcolor{orange}{Caren works as a teacher and she likes singing.} \textcolor{teal}{Emily is the colleague of Caren. Emily works as a} \\
         \midrule
         \multirow{2}{*}{Semantic paraphrasing} & a. \textcolor{orange}{Caren teaches students. } \textcolor{teal}{Emily is the colleague of Caren. Emily works as a} \\
         &b. \textcolor{orange}{The school hired Caren as a teacher.} \textcolor{teal}{Emily is the colleague of Caren. Emily works as a} \\
         \midrule
         \multirow{5}{*}{Renaming} &a. \textcolor{orange}{Anna works as a teacher.} \textcolor{teal}{Emily is the colleague of Anna. Emily works as a} \\
         &b. \textcolor{orange}{Tina works as a teacher.} \textcolor{teal}{Emily is the colleague of Tina. Emily works as a} \\ 
         &c. \textcolor{orange}{Emma works as a teacher.} \textcolor{teal}{Emily is the colleague of Emma. Emily works as a} \\ 
         &d. \textcolor{orange}{Cynthia works as a teacher.} \textcolor{teal}{Emily is the colleague of Cynthia. Emily works as a} \\ 
         &e. \textcolor{orange}{Kate works as a teacher.} \textcolor{teal}{Emily is the colleague of Kate. Emily works as a} \\ 
    \midrule
    \midrule
    \textbf{Original prompt}  & \textcolor{orange}{Sarah likes cats.} \textcolor{teal}{\underline{Sarah} \underline{would} \underline{like} \underline{to} \underline{have} \underline{a} \underline{pet}. \underline{She} \underline{will} \underline{get} a}\\
        \midrule
         \multirow{2}{*}{\makecell{Removing and adding \\irrelevant background}} & a. \textcolor{orange}{Sarah likes cats and Lisa likes dogs.} \textcolor{teal}{Sarah would like to have a pet. She will get a}\\
         & b. \textcolor{orange}{Sarah likes cats and cakes.} \textcolor{teal}{Sarah would like to have a pet. She will get a} \\
         \midrule
         \multirow{2}{*}{Semantic paraphrasing} & a. \textcolor{orange}{Sarah adores cats.} \textcolor{teal}{Sarah would like to have a pet. She will get a} \\
         &b. \textcolor{orange}{Sarah greatly enjoys the company of cats.} \textcolor{teal}{Sarah would like to have a pet. She will get a} \\
         \midrule
         \multirow{5}{*}{Renaming} &a. \textcolor{orange}{Anna likes cats.} \textcolor{teal}{Anna would like to have a pet. She will get a} \\
         &b. \textcolor{orange}{Tina likes cats.} \textcolor{teal}{Tina would like to have a pet. She will get a} \\ 
         &c. \textcolor{orange}{Emma likes cats.} \textcolor{teal}{Emma would like to have a pet. She will get a} \\ 
         &d. \textcolor{orange}{Cynthia likes cats.} \textcolor{teal}{Cynthia would like to have a pet. She will get a} \\ 
         &e. \textcolor{orange}{Emily likes cats.} \textcolor{teal}{Emily would like to have a pet. She will get a} \\ 
    \midrule
    \midrule
    \textbf{Original prompt}  & \textcolor{orange}{It takes three days by train from Greece to Paris. } \textcolor{teal}{\underline{Michael} \underline{departs} \underline{from} \underline{Greece} on \underline{Friday}, and \underline{he} \underline{will} \underline{arrive} \underline{in} \underline{Paris} on}\\
        \midrule
         \multirow{2}{*}{\makecell{Removing and adding \\irrelevant background}} & a. \textcolor{orange}{It takes three days by train from Greece to Paris and it only takes one day by train from Berlin to Paris.} \textcolor{teal}{Michael departs from Greece on Friday, and he will arrive in Paris on}\\
         & b. \textcolor{orange}{It takes three days by train from Greece to Paris and will pass by Vienna.} \textcolor{teal}{Michael departs from Greece on Friday, and he will arrive in Paris on} \\
         \midrule
         \multirow{2}{*}{Semantic paraphrasing} & a. \textcolor{orange}{It's a three-day train journey from Greece to Paris.} \textcolor{teal}{Michael departs from Greece on Friday, and he will arrive in Paris on} \\
         &b. \textcolor{orange}{The journey by train from Greece to Paris spans three days.} \textcolor{teal}{Michael departs from Greece on Friday, and he will arrive in Paris on} \\
         \midrule
         \multirow{5}{*}{Renaming} &a. \textcolor{orange}{It takes three days by train from Greece to Paris.} \textcolor{teal}{Tim departs from Greece on Friday, and he will arrive in Paris on} \\
         &b. \textcolor{orange}{It takes three days by train from Greece to Paris. } \textcolor{teal}{John departs from Greece on Friday, and he will arrive in Paris on} \\ 
         &c. \textcolor{orange}{It takes three days by train from Greece to Paris. } \textcolor{teal}{William departs from Greece on Friday, and he will arrive in Paris on} \\ 
         &d. \textcolor{orange}{It takes three days by train from Greece to Paris. } \textcolor{teal}{Jack departs from Greece on Friday, and he will arrive in Paris on} \\ 
         &e. \textcolor{orange}{It takes three days by train from Greece to Paris. } \textcolor{teal}{Peter departs from Greece on Friday, and he will arrive in Paris on} \\
    \midrule
    \midrule
   \textbf{Original prompt}  & \textcolor{orange}{Kate majors in English.} \textcolor{teal}{\underline{James} \underline{studies} \underline{the} \underline{same} \underline{subject} \underline{as} \underline{Kate}. \underline{James} \underline{majors} \underline{in}}\\
        \midrule
         \multirow{2}{*}{\makecell{Removing and adding \\irrelevant background}} & a. \textcolor{orange}{Kate majors in English and she plans to go hiking this weekend.} \textcolor{teal}{James studies the same subject as Kate. James majors in}\\
         & b. \textcolor{orange}{Kate majors in English and John majors in Chemistry.} \textcolor{teal}{James studies the same subject as Kate. James majors in} \\
         \midrule
         \multirow{2}{*}{Semantic paraphrasing} & a. \textcolor{orange}{English is Kate's major field of study.} \textcolor{teal}{James studies the same subject as Kate. James majors in} \\
         &b. \textcolor{orange}{Kate's area of study is English.} \textcolor{teal}{James studies the same subject as Kate. James majors in} \\
         \midrule
         \multirow{5}{*}{Renaming} &a. \textcolor{orange}{Anna majors in English. } \textcolor{teal}{James studies the same subject as Anna. James majors in} \\
         &b. \textcolor{orange}{Tina majors in English. } \textcolor{teal}{James studies the same subject as Tina. James majors in} \\ 
         &c. \textcolor{orange}{Emma majors in English.} \textcolor{teal}{James studies the same subject as Emma. James majors in} \\ 
         &d. \textcolor{orange}{Cynthia majors in English.} \textcolor{teal}{James studies the same subject as Cynthia. James majors in} \\ 
         &e. \textcolor{orange}{Emily majors in English.} \textcolor{teal}{James studies the same subject as Emily. James majors in} \\
    \midrule
    \midrule
    \textbf{Original prompt}  & \textcolor{orange}{Kate lives in Paris.} \textcolor{teal}{\underline{James} \underline{lives} \underline{in} the \underline{same} \underline{city} \underline{as} \underline{Kate}. \underline{James} \underline{lives} \underline{in}}\\
        \midrule
         \multirow{2}{*}{\makecell{Removing and adding \\irrelevant background}} & a. \textcolor{orange}{Kate lives in Paris and she likes dancing.} \textcolor{teal}{James lives in the same city as Kate. James lives in}\\
         & b. \textcolor{orange}{Kate lives in Paris and John lives in Seattle.} \textcolor{teal}{James lives in the same city as Kate. James lives in} \\
         \midrule
         \multirow{2}{*}{Semantic paraphrasing} & a. \textcolor{orange}{Kate resides in the city of Paris.} \textcolor{teal}{James lives in the same city as Kate. James lives in} \\
         &b. \textcolor{orange}{Kate's current residence is in Paris.} \textcolor{teal}{James lives in the same city as Kate. James lives in} \\
         \midrule
         \multirow{5}{*}{Renaming} &a. \textcolor{orange}{Anna lives in Paris.} \textcolor{teal}{James lives in the same city as Anna. James lives in} \\
         &b. \textcolor{orange}{Tina lives in Paris.} \textcolor{teal}{James lives in the same city as Tina. James lives in} \\ 
         &c. \textcolor{orange}{Emma lives in Paris.} \textcolor{teal}{James lives in the same city as Emma. James lives in} \\ 
         &d. \textcolor{orange}{Cynthia lives in Paris.} \textcolor{teal}{James lives in the same city as Cynthia. James lives in} \\ 
         &e. \textcolor{orange}{Emily lives in Paris.} \textcolor{teal}{James lives in the same city as Emily. James lives in} \\
    \bottomrule 
    \caption{List of prompts for the question-answer task. The premise is highlighted in orange and the question is highlighted in teal. For each prompt, there are 10 logically equivalent prompts including the original prompt. The selected 10 words are annotated in original prompts with underlines.}
    \label{tab:logical_equivalent_full}
\end{longtable}}

\subsection{Proof of universal matching properties of decomposed effects}\label{appendix:universal matching prove}

\begin{lemma}
    Given an LLM $v$ and an input prompt $\hat{\mathbf{x}}$, for any randomly masked sample $\hat{\mathbf{x}}_T, T\subseteq N$, we have
    \begin{equation}\label{eq:universal_effects}
    \begin{aligned}
        v(x_{n+1}|\hat{\mathbf{x}}_T) &=\sum\nolimits_{S\in \Omega_{\text{\rm and}}(\hat{\mathbf{x}}):\emptyset\neq S\subseteq T}\left(\mathcal{J}^{\text{\rm f}}_{\text{\rm and}}(S|\hat{\mathbf{x}})\!+\!\mathcal{J}^{\text{\rm c}}_{\text{\rm and}}(S|\hat{\mathbf{x}})\!+\!\mathcal{K}_{\text{\rm and}}(S|\hat{\mathbf{x}})\right) \\
        &+\sum\nolimits_{S\in \Omega_{\text{\rm or}}(\hat{\mathbf{x}}):S\cap T\neq \emptyset}\left(\mathcal{J}^{\text{\rm f}}_{\text{\rm or}}(S|\hat{\mathbf{x}}) \!+\!\mathcal{J}^{\text{\rm c}}_{\text{\rm or}}(S|\hat{\mathbf{x}}) \!+\! \mathcal{K}_{\text{\rm or}}(S|\hat{\mathbf{x}})\right) + v(x_{n+1}|\hat{\mathbf{x}}_\emptyset).
    \end{aligned}
    \end{equation}
\end{lemma}

\begin{proof}
    According to Equations~(\ref{eq:J_f}), (\ref{eq:J_c}) and (\ref{eq:K}), we have
    \begin{equation*}
    \begin{aligned}
        &\mathcal{J}^{\text{\rm f}}_{\text{\rm and}}(S|\hat{\mathbf{x}})\!+\!\mathcal{J}^{\text{\rm c}}_{\text{\rm and}}(S|\hat{\mathbf{x}})\!+\!\mathcal{K}_{\text{\rm and}}(S|\hat{\mathbf{x}}) \\
        = &I_{\text{and}}(S|\hat{\mathbf{x}}_q) + I_{\text{and}}(S|\hat{\mathbf{x}}) - I_{\text{and}}(S|\hat{\mathbb{X}}) + I_{\text{and}}(S|\hat{\mathbb{X}}) -  I_{\text{and}}(S|\hat{\mathbf{x}}_q) \\
        =& I_{\text {and}}(S|\hat{\mathbf{x}}),
    \end{aligned}  
    \end{equation*}
    and
        \begin{equation*}
    \begin{aligned}
        &\mathcal{J}^{\text{\rm f}}_{\text{\rm or}}(S|\hat{\mathbf{x}})\!+\!\mathcal{J}^{\text{\rm c}}_{\text{\rm or}}(S|\hat{\mathbf{x}})\!+\!\mathcal{K}_{\text{\rm or}}(S|\hat{\mathbf{x}}) \\
        = &I_{\text{or}}(S|\hat{\mathbf{x}}_q) + I_{\text{or}}(S|\hat{\mathbf{x}}) - I_{\text{or}}(S|\hat{\mathbb{X}}) + I_{\text{or}}(S|\hat{\mathbb{X}}) -  I_{\text{or}}(S|\hat{\mathbf{x}}_q) \\
        =& I_{\text {or}}(S|\hat{\mathbf{x}}).
    \end{aligned}  
    \end{equation*}
    Here, $\hat{\mathbb{X}} = \{\hat{\mathbf{x}}^1, \dots, \hat{\mathbf{x}}^k\}$.
    In this way, Equation~(\ref{eq:universal_effects}) can be rewritten as $v(x_{n+1}|\hat{\mathbf{x}}_T) =\sum\nolimits_{S\in \Omega_{\text{\rm and}}(\hat{\mathbf{x}}):\emptyset\neq S\subseteq T}I_{\text{and}}(S|\hat{\mathbf{x}}) 
        +\sum\nolimits_{S\in \Omega_{\text{\rm or}}(\hat{\mathbf{x}}):S\cap T\neq \emptyset}I_{\text{or}}(S|\hat{\mathbf{x}})  + v(x_{n+1}|\hat{\mathbf{x}}_\emptyset)$, which can be proven by Theorem~\ref{thm:universalMatching}.
\end{proof}
\subsection{Experiments on faithfulness of interaction explanations}\label{appendix:faithfulness}
The faithfulness of interaction-based explanations is guaranteed by the universal-matching property, see Theorem~\ref{thm:universalMatching}. To empirically verify whether the extracted interactions matched well with the LLM's output $v(x_{n+1}|\mathbf{x}_S)$, we followed~\cite{shen2023can} to measure the approximation errors. We used $\mathbf{v} = [v(x_{n+1}|\mathbf{x}_{S_1}), v(x_{n+1}|\mathbf{x}_{S_2}), \dots, v(x_{n+1}|\mathbf{x}_{S_{2^n}})] \in \mathbb{R}^{2^n}$ to denote the LLM's output vector on all $2^n$ masked samples, where the outputs were sorted in ascending order, \emph{i.e.}, $v(x_{n+1}|\mathbf{x}_{S_1})\leq v(x_{n+1}|\mathbf{x}_{S_2})\leq \dots \leq v(x_{n+1}|\mathbf{x}_{S_{2^n}})$. Then, we used $\mathbf{v}^{\text{approx}} = [v^{\text{approx}}(x_{n+1}|\mathbf{x}_{S_1}), v^{\text{approx}}(x_{n+1}|\mathbf{x}_{S_2}), \dots, v^{\text{approx}}(x_{n+1}|\mathbf{x}_{S_{2^n}})] \in \mathbb{R}^{2^n}$ to represent the output vector approximated by decomposed effects (see Equation~(\ref{eq:universal_effects})). In this way, for each $i$-th masked sentence, $e_i = \vert \mathbf{v}(x_{n+1}|\mathbf{x}_{S_i}) - \mathbf{v}^{\text{approx}}(x_{n+1}|\mathbf{x}_{S_i})\vert$ represented the matching error. 

Specifically, we used top-ranked $50, 100, 150,$ and $200$ interactions and measured the matching errors. Figure~\ref{fig:fig11} shows LLaMA's confidence scores on all $2^n$ masked samples in ascending order. We averaged the matching errors of every 50 neighboring masked samples to provide a smoothed visualization of the matching error. We can see that the confidence scores on all masked samples can be well approximated by the decomposed effects.

\begin{figure}[t]
    \centering
    \includegraphics[clip, trim=2cm 0cm 0cm 0cm, width=0.9\textwidth]{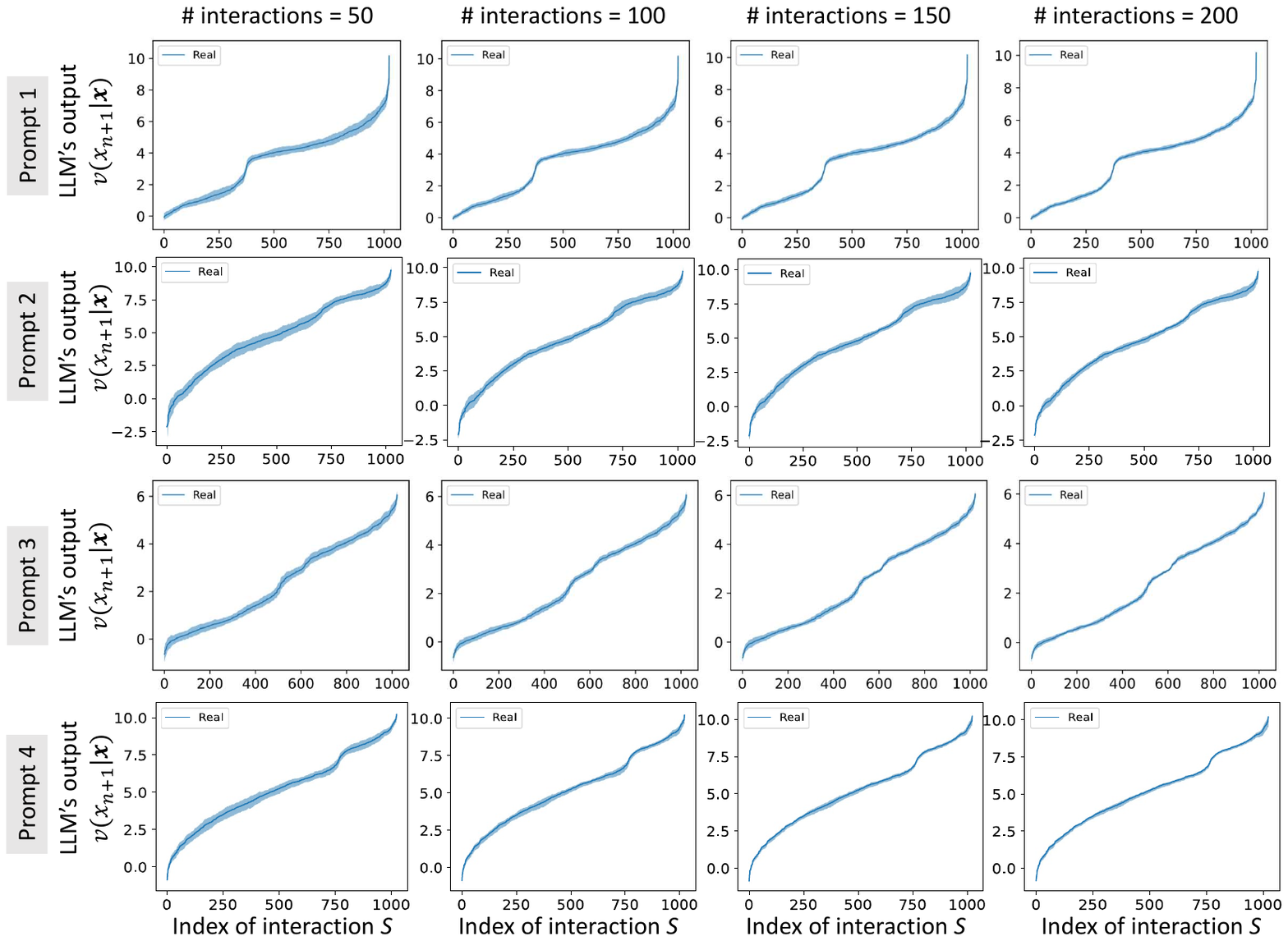}
    \caption{LLM's confidence score on differently masked samples (sorted in an ascending order). The shade area represents the matching error of using the top-ranked salience interactions to match the real LLM's output.}
    \label{fig:fig11}
\end{figure}

\subsection{Experimental details} \label{appendix:experimentaldetails}
We conducted experiments on three LLMs: OPT-1.3B~\cite{zhang2022opt}, LLaMA-7B~\cite{touvron2023llama}, and GPT-3.5-Turbo~\cite{brown2020language}, using the question-answer samples listed in Table~\ref{tab:logical_equivalent_full}. The experiments used publicly available pre-trained LLMs and were run on an NVIDIA A800 GPU.

\textbf{Selecting input variables.} Theoretically, given an input prompt with $n$ input variables (either tokens or words, in this paper, we use words to facilitate semantic analysis), the time complexity of extracting AND-OR interactions is $2^n$ (requiring inference on $2^n$ masked samples). In real-world applications, input prompts often contain a large number of words,  making such computations infeasible. To address this issue, we followed the method described in ~\cite{shen2023can} to select a set of words as input variables and treat the remaining words as a constant background. This approach allows for the extraction of AND-OR interactions among the selected words. Specifically, we selected 10 words in the question of each sample, and they were underlined in the original prompts in Table~\ref{tab:logical_equivalent_full}. For those prompts longer than 10 words, we excluded some words with no clear semantic meanings, such as articles, prepositions, and conjunctions.

Notably, the $10$ input words annotated in the original prompt should also be identified and annotated in the remaining logically equivalent prompts. For the ``Renaming'' category, while we changed the names of entities involved in the reasoning process, we annotated the corresponding new names as replacements for the originally annotated words.

\subsection{More experiment results} \label{appendix:result}

\textbf{The ratio of reasoning effects and the ratio of chaotic memorization effects.} Here, we reported the standard deviations of the ratio of reasoning effects $\rho_{\text{r}}$ and the ratio of chaotic memorization effects $\rho_{\text{c}}$ in Table~\ref{tab:ratio2}.

  \begin{table}[t]
  \begin{minipage}[c]{0.57\textwidth}
  \centering
  {\tiny
    \begin{tabular}{c ccc}
    \toprule
    &OPT-1.3B & LLaMA-7B & GPT-3.5-Turbo \\
    \midrule
    $\rho_{\text{r}} \uparrow$ & $39.12\% \pm 7.04\%$  &  $42.82\% \pm 3.30\%$& $\mathbf{45.24\%}\pm 7.09\%$ \\
    \midrule
    $\rho_{\text{c}} \downarrow$ & $\mathbf{7.37}\% \pm 3.65\%$  &  $7.57\%\pm 2.73\%$& $7.81\% \pm 4.00\%$ \\
    \bottomrule 
    \end{tabular}}
      \end{minipage}
      \begin{minipage}[c]{0.4\textwidth}
          \caption{The ratio $\rho_{\text{r}}$ of in-context reasoning effects and the ratio $\rho_{\text{c}}$ of chaotic memorization effects (with standard deviations).}
    \label{tab:ratio2}
      \end{minipage}
\end{table} 
\textbf{Visualization of the decomposed effects on a single sample.} Figures~\ref{fig:fig8}, \ref{fig:fig9} and \ref{fig:fig10} visualize the foundational memorization effects, chaotic memorization effects and reasoning effects over different orders $m$, extracted from three LLMs on three different prompts, respectively. Additionally, the reasoning effects were categorized into enhanced inference patterns, eliminated inference patterns and reversed inference patterns.
\begin{figure}[t]
    \centering
    \includegraphics[clip, trim=0cm 5cm 0cm 0cm, width=\textwidth]{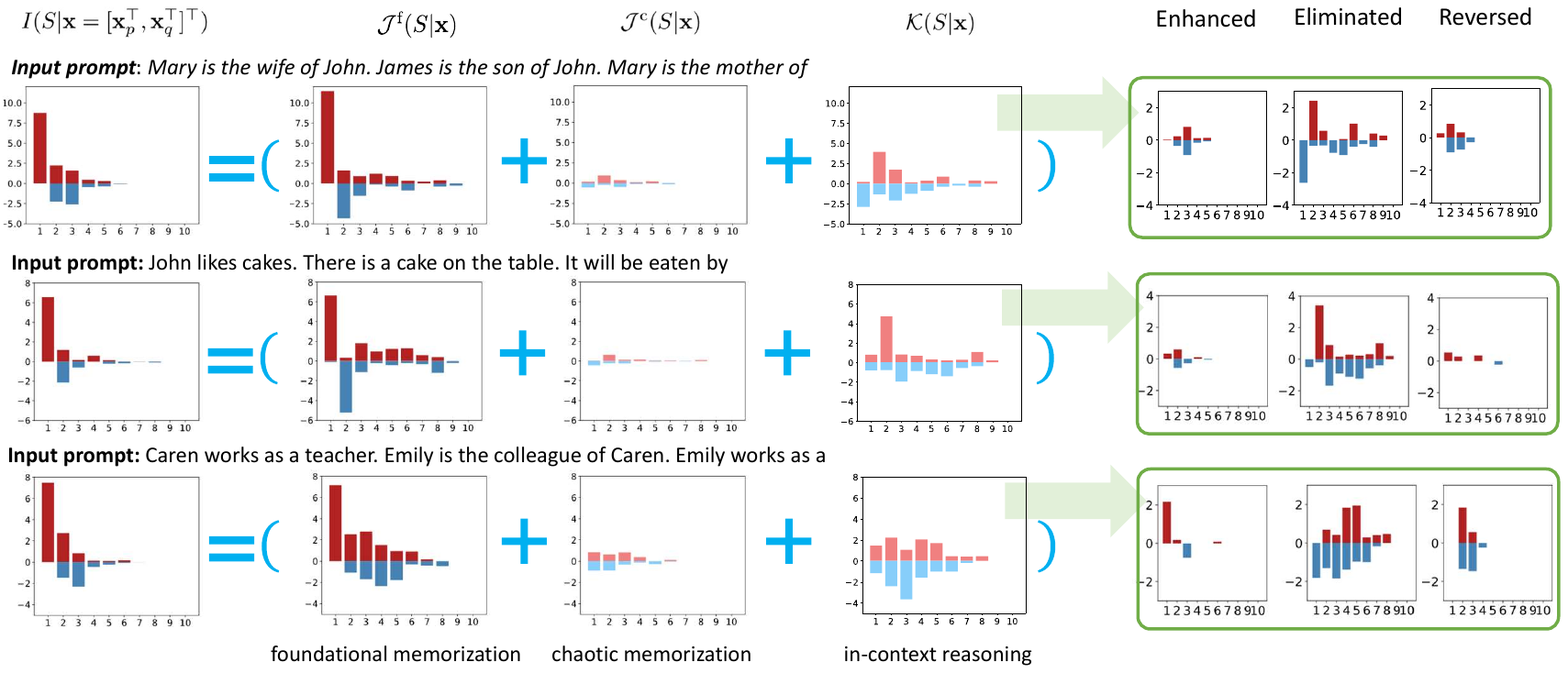}
    \caption{Distribution of the strength of foundational memorization interactions, chaotic memorization interactions and in-context reasoning interactions over different orders $m$ extracted from LLaMA-7B on a single prompt.}
    \label{fig:fig8}
\end{figure}

\begin{figure}[t]
    \centering
    \includegraphics[clip, trim=0cm 5cm 0cm 0cm, width=\textwidth]{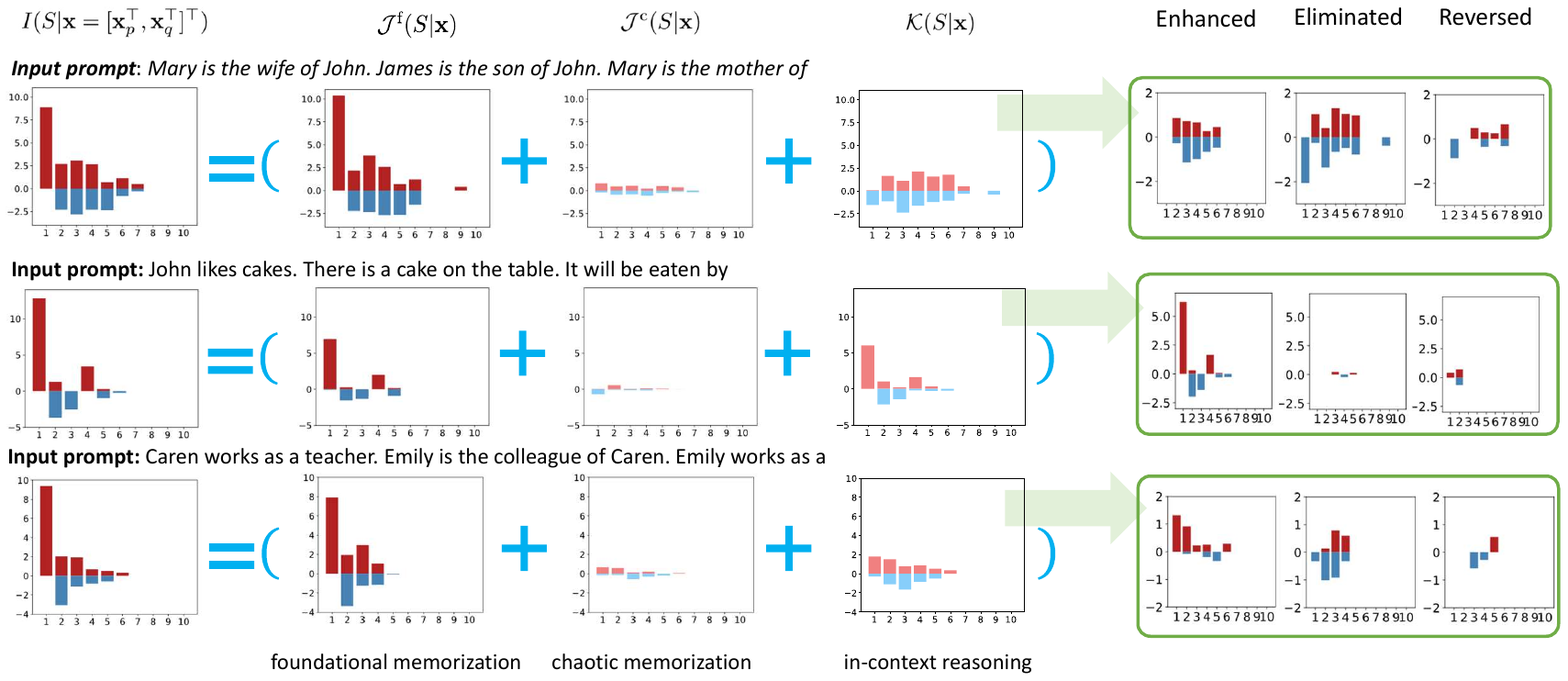}
    \caption{Distribution of the strength of foundational memorization interactions, chaotic memorization interactions and in-context reasoning interactions over different orders $m$ extracted from OPT-1.3B on a single prompt.}
    \label{fig:fig9}
\end{figure}

\begin{figure}[t]
    \centering
    \includegraphics[clip, trim=0cm 5cm 0cm 0cm, width=\textwidth]{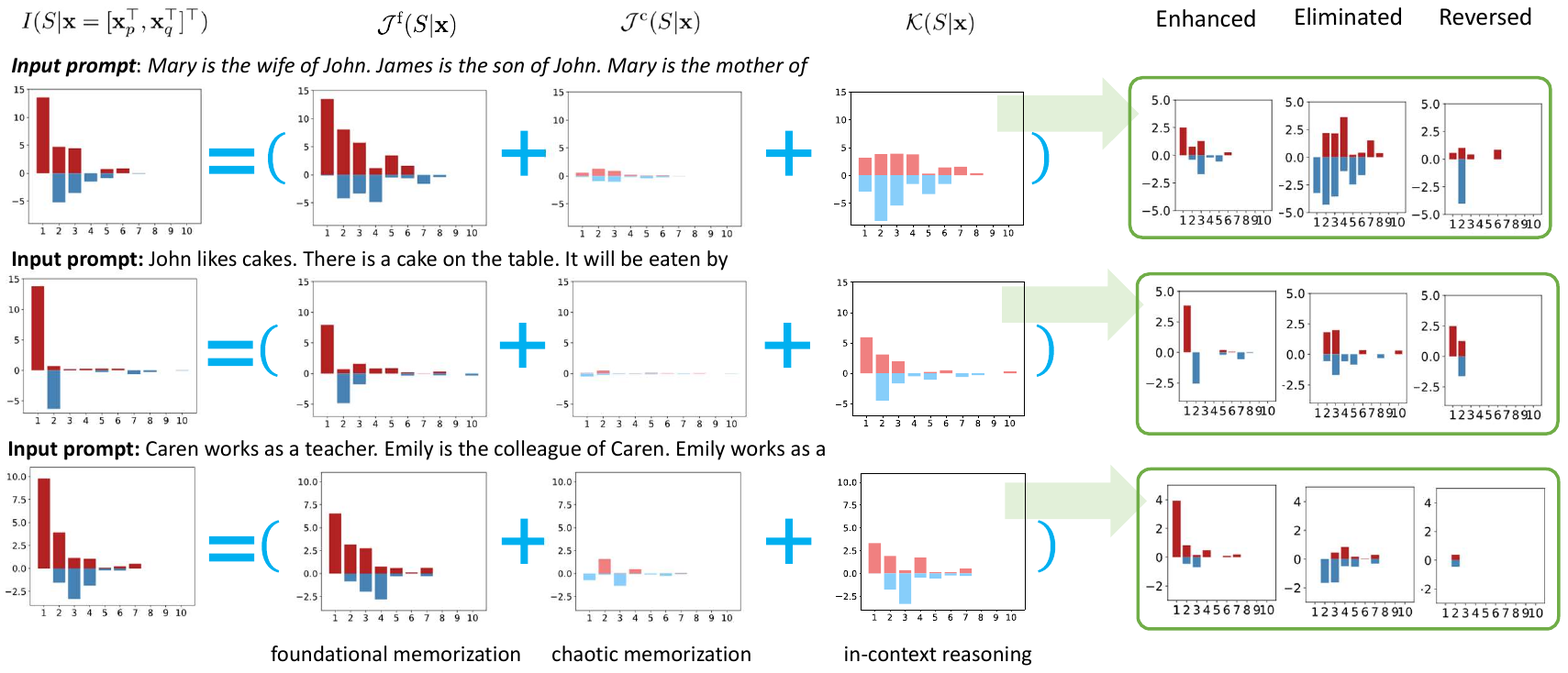}
    \caption{Distribution of the strength of foundational memorization interactions, chaotic memorization interactions and in-context reasoning interactions over different orders $m$ extracted from GPT-3.5-Turbo on a single prompt.}
    \label{fig:fig10}
\end{figure}

\end{document}